\documentclass[12pt]{article}
\usepackage{graphicx} 
\usepackage{amsmath, amssymb, amsthm,wasysym}
\usepackage{algorithm}
\usepackage{algorithmic}
\usepackage{color}
\usepackage{times}
\usepackage{times}

\newtheorem{theorem}{Theorem}   
\newtheorem{lemma}[theorem]{Lemma}

\newtheorem{remark}{Remark}

\newfont{\msym}{msbm10}
\newcommand{\reals}{\mathbb{R}}

\newcommand{\mcal}[1]{{\mathcal{#1}}}

\newcommand{\nolineskips}{
\setlength{\parskip}{0pt}
\setlength{\parsep}{0pt}
\setlength{\topsep}{0pt}
\setlength{\partopsep}{0pt}
\setlength{\itemsep}{0pt}}

\newcommand{\beq}[1]{\begin{equation}\label{#1}}
\newcommand{\eeq}{\end{equation}}
\newcommand{\beqa}{\begin{eqnarray}}
\newcommand{\eeqa}{\end{eqnarray}}

\renewcommand{\eqref}[1]{Equation~(\ref{eq:#1})}

\newcommand{\mb}[1]{{\boldsymbol{#1}}}

\newcommand{\vx}{\mb{x}} 
\newcommand{\vxi}[1]{\vx_{#1}}

\newcommand{\R}{\mcal{R}}

\newcommand{\newstufffroma}[1]{}

\newcommand{\newstufffrom}[1]{}

\newcommand{\oldnote}[2]{}

\newcommand{\commentout}[1]{}

\newcounter {mySubCounter}
\newcommand {\twocoleqn}[4]{
  \setcounter {mySubCounter}{0} %
  \let\OldTheEquation \theequation %
  \renewcommand {\theequation }{\OldTheEquation \alph {mySubCounter}}%
  \noindent \hfill%
  \begin{minipage}{.40\textwidth}
\vspace{-0.6cm}
    \begin{equation}\refstepcounter{mySubCounter}
      #1 
    \end {equation}
  \end {minipage}
~~~~~~
  \addtocounter {equation}{ -1}%
  \begin{minipage}{.40\textwidth}
\vspace{-0.6cm}
    \begin{equation}\refstepcounter{mySubCounter}
      #3 
    \end{equation}
  \end{minipage}%
  \let\theequation\OldTheEquation
}

\newcommand{\bx}{\boldsymbol{x}}
\newcommand{\hY}{{\hat Y}}
\newcommand{\bu}{\boldsymbol{u}}
\newcommand{\bw}{\boldsymbol{w}}

\newcommand{\E}{\ensuremath{\mathbb{E}}}
\renewcommand{\Pr}{\mathbb{P}}
\newcommand{\hDelta}{\widehat {\Delta}}

\newcommand{\hp}{\widehat {p}}
\newcommand{\hf}{\widehat {f}}
\newcommand{\hr}{\widehat {r}}
\newcommand{\ha}{\widehat {a}}

\newcommand{\sfrac}[2]{\mbox{$\frac{#1}{#2}$}}

\begin{document} 

\title{{\bf On Multilabel Classification and Ranking with Partial Feedback, Ver. 3}}
\author{
Claudio Gentile\\
Universita' dell'Insubria\\
\texttt{claudio.gentile@uninsubria.it}
\and
Francesco Orabona \\
Toyota Technological Institute at Chicago\\
\texttt{francesco@orabona.com}}
\maketitle

\begin{abstract}
We present a novel multilabel/ranking algorithm working in partial information
settings. The algorithm is based on 2nd-order descent methods, and
relies on upper-confidence bounds to trade-off exploration and
exploitation. 
We analyze this algorithm in a partial adversarial setting, where 
covariates can be adversarial, but multilabel probabilities are ruled 
by (generalized) linear models. We show $O(T^{1/2}\log T)$ regret bounds,
which improve in several ways on the existing results.
We test the effectiveness of our upper-confidence scheme by contrasting 
against full-information baselines on real-world multilabel datasets,
often obtaining comparable performance.
\end{abstract}

\section{Introduction}
Consider a book recommendation system. Given a customer's profile, the system 
recommends a few possible books to the user by means of, e.g., a limited number of banners 
placed at different positions on a webpage. The system's goal is to select books that 
the user likes and possibly purchases. Typical feedback in such systems is the actual 
action of the user or, in particular, what books he has bought/preferred, if any. The
system cannot observe what would have been the user's actions had other books got
recommended, or had the same book ads been placed in a different order within the webpage. 

Such problems are collectively referred to as learning with partial feedback. As 
opposed to the full information case, where the system (the learning algorithm) knows 
the outcome of each possible response (e.g., the user's action for each and every 
possible book recommendation placed in the largest banner ad), 
in the partial feedback setting, the system only observes the response to very limited 
options and, specifically, the option that was actually recommended.

In this and many other examples of this sort, it is reasonable to assume that 
recommended options are not given the same treatment by the system, e.g., 
large banners which are displayed on top of the page should somehow be more 
committing as a recommendation than smaller ones placed elsewhere. Moreover, 
it is often plausible to interpret the user feedback as a preference (if any) 
{\em restricted to} the displayed alternatives.

In this paper, we consider instantiations of this problem in the multilabel and learning-to-rank
settings. Learning proceeds in rounds, in each time step $t$ the 
algorithm receives an instance $\bx_t$ and outputs an ordered subset $\hY_t$ of labels 
from a finite set of possible labels $[K] = \{1, 2, \ldots, K\}$. Restrictions might 
apply to the size of $\hY_t$ (due, e.g., to the number of available slots in the 
webpage). The set $\hY_t$ corresponds to the aforementioned recommendations, and 
is intended to approximate the true set of preferences 
associated with $\bx_t$. However, the latter set is never observed. In its stead, the 
algorithm receives $Y_t \cap \hY_t$, where $Y_t \subseteq [K]$ is a {\em noisy version} 
of the true set of user preferences on $\bx_t$. 
When we are restricted to $|\hY_t| = 1$ for all $t$,
this becomes a multiclass classification problem with bandit feedback -- see below.


\subsection{Related work}
This paper lies at the intersection between online learning with partial feedback
and multilabel classification/ranking. Both fields include a substantial amount of
work, so we can hardly do it justice here. In the sequel, we outline some of the
main contributions in the two fields, with an emphasis on 
those we believe are the most related to this paper.

A well-known tool for facing the problem of partial feedback in online learning is to 
trade off exploration and exploitation through upper confidence bounds. 
This technique has been introduced by \cite{lr85}, and can by now be considered a 
standard tool. In the so-called {\em bandit} setting with contextual information
(sometimes called bandits with side information or bandits with covariates, e.g., 
\cite{a03,dhk08,fcgs10,cg11,ko11}, and references therein) an online algorithm 
receives at each time step a {\em context} (typically, in the form of a feature vector $\bx$) 
and is compelled to select an action (e.g., a label), whose goodness is quantified by a predefined loss function. 
Full information about the loss function 
(one that would perhaps allow to minimizes the total loss over the contexts seen so far)
is not available.  
The specifics of the interaction model determines which pieces of loss 
will be observed by the algorithm, e.g., the actual value of the loss on the 
chosen action, some information on more profitable directions on the action space,
noisy versions thereof, etc.
The overall goal is to compete against classes of functions that map contexts to 
(expected) losses in a regret sense, that is, to obtain {\em sublinear} 
cumulative regret bounds. 

All these algorithms share the common need to somehow trade off an exploratory 
attitude for gathering 
loss information on unchosen directions of the context-action space, and an exploitatory
attitude for choosing actions that are deemed best according to the available data. 
For instance, \cite{a03,dhk08,fcgs10,ayps11} work in a finite action space where the mappings
context-to-loss for each action are linear (or generalized linear, as \cite{fcgs10}'s) 
functions of the features. They all obtain $T^{1/2}$-like regret bounds, where $T$ is the time
horizon. This is extended 
by \cite{ko11}, where the loss function is modeled as a sample from a Gaussian process
over the joint context-action space. We are using a similar (generalized) linear modeling
here.
An earlier (but somehow more general) setting that models such mappings by VC-classes 
is considered by \cite{lz07}, where a $T^{2/3}$ regret bound has been proven under 
i.i.d. assumptions. 
Linear multiclass classification problems with bandit feedback 
are considered by, e.g., \cite{kakade08banditron,cg11,hk11}, where either 
$T^{2/3}$ or $T^{1/2}$ or even logarithmic regret bounds are proven, depending on the
noise model and the underlying loss functions.

All the above papers do not consider {\em structured} action spaces, where the learner
is afforded to select {\em sets} of actions, which is more suitable
to multilabel and ranking problems. Along these lines are
\cite{hk09,sgk09,krs10,srg10,sj12,aks11}. The general problem of 
online minimization of a submodular loss function under both full and bandit information
without covariates is considered by \cite{hk09}, achieving a regret $T^{2/3}$ in the bandit case.
\cite{sgk09} consider the problem of online learning of assignments, where
at each round an algorithm is requested to assign positions (e.g., rankings) to sets of items
(e.g., ads) with given constraints on the set of items that can be placed in each position.
Their problem shares similar motivations as ours but, again, the bandit version of their
algorithm does not explicitly take side information into account, and leads to a $T^{2/3}$ 
regret bound.
Another paper with similar goals but a different mathematical model 
is by \cite{krs10}, where the aim is to learn a suitable ordering (an ``ordered slate") of the available 
actions. Among other things, the authors prove a 
$T^{1/2}$ regret bound in the bandit setting with a multiplicative weight updating scheme.
Yet, no contextual information is incorporated. \cite{srg10} motivate the ability of selecting sets of actions
by a problem of diverse retrieval in large document collections which are
meant to live in a general metric space. 
In contrast to our paper, that approach does not lead to strong regret guarantees for specific
(e.g., smooth) loss functions. 
\cite{sj12} use a simple linear model for the hidden utility function of 
users interacting with a web system and providing partial feedback in any form
that allows the system to make significant progress in learning this function 
(this is called an $\alpha$-informative feedback by the authors).
Under these assumptions, a regret bound of $T^{1/2}$ is again provided that depends 
on the degree of informativeness of the feedback, as measured by the progress made 
during the learning process.
It is experimentally argued that this feedback is typically made available by a user 
that clicks on relevant URLs out of a list presented by a search engine. 
Despite the neatness of the argument, no formal effort is put into relating this 
information to the context information at hand or, more generally,
to the way data are generated.
The recent paper \cite{aks11} investigates classes of graphical models for
contextual bandit settings that afford richer interaction between contexts and actions
leading again to a $T^{2/3}$ regret bound.

Finally, a very interesting recent work that came to our attention at the time 
of writing this extended version of our conference paper \cite{go12} is \cite{bs12}. 
In that paper, the authors
provide sufficient conditions that insure rates of the form $T^{1/2}$ 
in partial monitoring
games with side information. Partial monitoring is an attempt to formalize 
through a unifying language the partial information settings where the algorithm
is observing only partial information about the loss of its action, in the form
of some kind of feedback or ``signal".
The results presented by \cite{bs12} do not seem to conveniently
extend to the structured action space setting we are interested in (or, if 
they do, we do not see it in the current version of their paper). 
Moreover, being very general in scope, 
that paper is missing a tight dependence of the regret bound on the number of 
available actions, which can be very large in structured action spaces.



The literature on multilabel learning and learning to rank is overwhelming.
The wide attention
this literature attracts is often motivated by its web-search-engine 
or recommender-system applications, and many of the papers are experimental in nature.
Relevant references include
\cite{tkv11,fhlmb08,dwch12}, along with references therein.  
Moreover, when dealing with multilabel, the typical assumption is full supervision,
an important concern being modeling correlations among classes. In contrast to that, 
the specific setting we are considering here
need not face such a modeling \cite{dwch12}.
The more recent work \cite{wkpj12} 
reduces any online algorithm working on pairwise loss functions 
(like a ranking loss) to a batch algorithm with generalization bound guarantees. But, 
again, only fully supervised settings are considered.
Other related references are \cite{hgo00,fiss03}, 
where learning is by pairs of
examples. Yet, these approaches need i.i.d. assumptions on the data, and typically
deliver batch learning procedures. 

To summarize, whereas we are technically closer to the linear modeling approaches
by \cite{a03,dhk08,dgs10,cg11,fcgs10,ayps11,ko11,bs12}, from a 
motivational standpoint we are perhaps closest to \cite{sgk09,krs10,sj12}.

\subsection{Our results}
We investigate the multilabel and learning-to-rank problems in a partial feedback
scenario with contextual information, where we assume a probabilistic linear model
over the labels, although the contexts can be chosen by an adaptive adversary.
We consider two families of loss functions, one is a cost-sensitive multilabel loss that generalizes
the standard Hamming loss in several respects, the other is a kind of (unnormalized) ranking loss. 
In both cases, the learning algorithm is maintaining a (generalized) linear predictor 
for the probability that a given label occurs, the ranking being produced by 
upper confidence-corrected estimated probabilities.
In such settings, we prove $T^{1/2}\log T$ cumulative regret bounds, which
are essentially optimal (up to log factors) in some cases.
A distinguishing feature of our user feedback model is that, unlike previous 
papers (e.g., \cite{hk09,sgk09,ayps11,ko11}), we are not assuming the algorithm
is observing a noisy version of the risk function on the currently selected action. 
In fact, when a generalized linear model is adopted,
the mapping context-to-risk turns out to be nonconvex in the parameter space. 
Furthermore, when operating on structured action spaces this more traditional 
form of bandit model does not seem appropriate to capture
the typical user preference feedback.
Our approach is based on having the loss decouple from the label generating model,
the user feedback being a noisy version of the gradient of a {\em surrogate} convex loss
associated with the model itself. As a consequence, the algorithm is not directly 
dealing with the original loss when making exploration. 
In this sense, we are more 
similar to the multiclass bandit algorithm by \cite{cg11}. Yet, our work
is a substantial departure from \cite{cg11}'s in that we lift their machinery 
to nontrivial structured action spaces, and we do so by means of generalized linear models.
On one hand, these extensions pose several extra technical challenges; 
on the other, they provide additional modeling power and practical advantage.

Though the emphasis is on theoretical results, we also validate our algorithms on 
two real-world multilabel datasets w.r.t. a number
of loss functions,
showing good comparative performance against simple multilabel/ranking baselines 
that operate with full information.

\subsection{Structure of the paper}
The paper is organized as follows. In Section \ref{s:model} we introduce our
learning model, our first loss function, the label generation model, and some
preliminary results and notation used throughout the rest of the paper.
In Section \ref{s:alg} we describe our partial feedback algorithm working
under the loss function introduced in Section \ref{s:model}, along with
the associated regret analysis. In Section \ref{s:rank} we show that
a very similar machinery applies to ranking with partial feedback, where
the loss function is a kind of pairwise ranking loss (with partial feedback).
Similar regret bounds are then presented that work under additional modeling
restrictions. In Section \ref{s:exp} we provide
our experimental evidence comparing our method with its immediate
full information counterpart. Section \ref{s:tech} gives proof ideas and
technical details. The paper is concluded with Section \ref{s:concl}, where
possible directions for future research are mentioned.

\section{Model and preliminaries}\label{s:model}
We consider a setting where the algorithm receives at time $t$ the side
information vector $\vxi{t} \in \reals^d$, is allowed to output a 
(possibly ordered) subset\footnote
{
An ordered subset is like a list with {\em no repeated} items.
}
$\hY_t \subseteq [K]$ of the set of possible
labels, then the subset of labels $Y_t \subseteq [K]$ associated with
$\vxi{t}$ is generated, and the algorithm gets as feedback
$\hY_t \cap Y_t$.
The loss suffered by the algorithm may take into account several things: 
the {\em distance} between $Y_t $ and $\hY_t$ (both viewed as sets), as well as 
the {\em cost} for playing $\hY_t$. The cost $c(\hY_t)$ associated with 
$\hY_t$ might be given by the sum of costs suffered
on each class $i \in \hY_t$, where we possibly take into account the {\em order}
in which $i$ occurs within $\hY_t$ (viewed as an ordered list of labels).
Specifically, given constant $a \in [0,1]$ and costs 
$c = \{c(i,s), i = 1, \ldots, s, s \in [K] \}$, such that
$1 \geq c(1,s) \geq c(2,s) \geq \ldots c(s,s) \geq 0$, for all $s \in [K]$, 
we consider the loss function
\[
\ell_{a,c}(Y_t,\hY_t) = a\,|Y_t\setminus \hY_t| + (1-a)\,\sum_{i \in \hY_t\setminus Y_t} c(j_i, |\hY_t|),
\]
where $j_i$ is the position of class $i$ in $\hY_t$, and $c(j_i,\cdot)$ depends on $\hY_t$ only through its size  $|\hY_t|$. 
In the above, the first term accounts for the false negative mistakes, hence there is no
specific ordering of labels therein. The second term collects the loss
contribution provided by all false positive classes, taking into account through
the costs $c(j_i,|\hY_t|)$ the order in which labels occur in $\hY_t$. The constant
$a$ serves as weighting the relative importance of false positive
vs. false negative mistakes\footnote{Notice that $a$ is not redundant here, 
since the costs $c(i,s)$ have been normalized to [0,1].}.
As a specific example, suppose that $K = 10$, the costs $c(i,s)$ are given by
$c(i,s)  = (s-i+1)/s, i = 1, \ldots, s$, the algorithm 
plays $\hY_t = (4,3,6)$, but $Y_t$ is $\{1,3,8\}$. In this case,  
$|Y_t\setminus \hY_t| = 2$, and 
$\sum_{i \in \hY_t\setminus Y_t} c(j_i,|\hY_t|) = 3/3 + 1/3$, i.e., 
the cost for mistakingly playing class 4 in the top slot of $\hY_t$ is more damaging than mistakingly playing class 6 in the third slot.
In the special case when all costs are unitary, there is no longer need to view
$\hY_t$ as an ordered collection, and the above loss reduces to a standard Hamming-like 
loss between sets $Y_t$ and $\hY_t$, i.e., 
$a\,|Y_t\setminus \hY_t| + (1-a)\,|\hY_t\setminus Y_t|$.
Notice that the partial feedback $\hY_t \cap Y_t$ allows the algorithm to know which of the 
chosen classes in $\hY_t$ are good or bad (and to what extent, because of the selected ordering within $\hY_t$). Yet, the algorithm 
does not observe the value of $\ell_{a,c}(Y_t,\hY_t)$ bacause $Y_t \setminus \hY_t$ 
remains hidden.

The reader should also observe the asymmetry between the label set $\hY_t$ 
produced by the algorithm
and the true label set $Y_t$: The algorithm predicts an ordered set of labels,
but the true set of labels is unordered. In fact, it is often the case in, e.g.,
recommender system practice, that the user feedback does not contain preference
information in the form of an ordered set of items. Still, in such systems
we would like to get back to the user with an appropriate ranking over the items.
%
%

Working with the above loss function makes the algorithm's output $\hY_t$ become a 
ranked list of classes, where ranking is {\em restricted} 
to the deemed relevant classes only.
In this sense, the above problem can be seen as a partial information version of the 
multilabel ranking problem (see \cite{fhlmb08}, and references therein). In a standard multilabel
ranking problem a classifier has to provide for any given instance $\bx_t$, both a 
separation between relevant and irrelevant classes and a ranking of the classes within
the two sets (or, perhaps, over the whole set of classes, as long as ranking is consistent
with the relevance separation). 
In our setting, instead, ranking applies to the selected classes only, 
but the information gathered by the algorithm while training is partial. That is, 
only a relevance feedback among the selected classes is observed (the set $Y_t\cap \hY_t$), 
but no supervised ranking information (e.g., in the form of pairwise preferences) 
is provided to the algorithm within this set. 
Alternatively, we can think of a ranking framework where restrictions
on the size of $\hY_t$ are set by an exogenous (and possibly time-varying) 
parameter of the problem, and the algorithm is required to
provide a ranking complying with these restrictions.

Another important concern we would like to address with our loss function $\ell_{a,c}$
is to avoid combinatorial explosions due to the exponential number of possible choices
for $\hY_t$. As we shall see below, this is guaranteed
by the chosen structure for costs $c(i,s)$. Another loss function providing similar
guarantees (though with additional modeling restrictions) is the (pairwise) ranking 
loss considered in Section \ref{s:rank},
where more on the connection to the ranking setting with partial feedback is given.

The problem arises as to which noise model
we should adopt so as to encompass significant real-world settings while at the same time
affording {\em efficient implementation} of the resulting algorithms.
For any subset $Y_t \subseteq [K]$, we let $(y_{1,t}, \ldots, y_{K,t}) \in \{0,1\}^K$ be 
the corresponding indicator vector. Then it is easy to see that
\begin{align*}
\ell_{a,c}(Y_t,\hY_t) 
&=  
a\,\sum_{i \notin \hY_t} y_{i,t} + (1-a)\,\sum_{i \in \hY_t} c(j_i, |\hY_t|)\,(1- y_{i,t})\\
&=
a\,\sum_{i = 1}^K y_{i,t} 
+ 
(1-a)\,\sum_{i \in \hY_t} \left( c(j_i,|\hY_t|) - \left(\sfrac{a}{1-a} + c(j_i,|\hY_t|)\right)\,y_{i,t} \right)~.
\end{align*}
Moreover, because the first sum does not depend on $\hY_t$, for the sake of optimizing over $\hY_t$
(but also for the sake of defining the regret $R_T$ -- see below) we can equivalently define
\begin{equation}\label{e:symmdiff}
\ell_{a,c}(Y_t,\hY_t) 
= (1-a)\,\sum_{i \in \hY_t} \left( c(j_i,|\hY_t|) - \left(\sfrac{a}{1-a} + c(j_i,|\hY_t|)\right)\,y_{i,t} \right)~.
\end{equation}
Let $\Pr_t(\cdot)$ be a shorthand for the conditional probability $\Pr_t(\cdot\,|\,\bx_t)$,
where the side information vector $\bx_t$ can in principle be generated by an adaptive adversary
as a function of the past. 
Then 
\[
\Pr_t(y_{1,t}, \ldots, y_{K,t}) = \Pr(y_{1,t}, \ldots, y_{K,t}\,|\, \bx_t),
\]
where the marginals $\Pr_t(y_{i,t}= 1)$ satisfy\footnote
{
The reader familiar with generalized linear models will recognize 
the derivative
of the function $p(\Delta) = \frac{g(-\Delta)}{g(\Delta)+g(-\Delta)}$ as
the (inverse) link function
of the associated canonical exponential family of distributions
\cite{mcn89}.
}
\begin{equation}\label{e:labgenmult}
\Pr_t(y_{i,t}= 1)  
= \frac{g(-\bu_{i}^\top\bx_t)}{g(\bu_{i}^\top\bx_t)+g(-\bu_{i}^\top\bx_t)}, \qquad i = 1, \ldots, K,
\end{equation}
for some $K$ vectors $\bu_1, \dots, \bu_K\in \R^d$ and some (known) function 
$g\,:\,D \subseteq \R \rightarrow \R^+$. The model is well defined if $\bu_{i}^\top\bx \in D$
for all $i$ and all $\bx \in \R^d$ chosen by the adversary.
We assume for the sake of simplicity that $||\bx_t|| = 1$ for all $t$.
Notice that here the variables $y_{i,t}$ {\em need not} be conditionally independent.
We are only definining a family of allowed joint distributions  $\Pr_t(y_{1,t}, \ldots, y_{K,t})$
through the properties of their marginals $\Pr_t(y_{i,t})$. 
A classical result in the theory of
copulas \cite{sk59} makes one derive all allowed joint distributions starting from 
the corresponding one-dimensional marginals.

The function $g$ above will be instantiated to the negative derivative of a suitable
convex and nonincreasing loss function $L$ which our algorithm will be based upon. 
For instance, if $L$ is the square loss
$L(\Delta) = (1-\Delta)^2/2$, then $g(\Delta) = 1-\Delta$, resulting in 
$\Pr_t(y_{i,t}= 1) = (1+\bu_{i}^\top\bx_t)/2$, under the assumption $D = [-1,1]$.
If $L$ is the logistic loss $L(\Delta) = \ln(1+e^{-\Delta})$, then
$g(\Delta) = \frac{1}{e^{\Delta}+1}$, and 
$\Pr_t(y_{i,t}= 1) = e^{\bu_{i}^\top\bx_t}/(e^{\bu_{i}^\top\bx_t}+1)$,
with domain $D = \R$.
Observe that in both cases $\Pr_t(y_{i,t}= 1)$ is an increasing
function of $\bu_{i}^\top\bx_t$. This will be true in general.

Set for brevity $\Delta_{i,t} = \bu_{i}^\top\bx_t$.
Taking into account (\ref{e:symmdiff}), this model allows us to write the 
(conditional) expected loss of the algorithm playing $\hY_t$ as
\begin{equation}\label{e:expectedloss}
\E_t[\ell_{a,c}(Y_t,\hY_t)] 
=  
(1-a)\,\sum_{i \in \hY_t} \left( c(j_i,|\hY_t|) - \left(\sfrac{a}{1-a} + c(j_i,|\hY_t|)\right)\,p_{i,t} \right)~,
\end{equation}
where we introduced the shorthands
\[
p_{i,t} = p(\Delta_{i,t}),\qquad\qquad 
p(\Delta) = \frac{g(-\Delta)}{g(\Delta)+g(-\Delta)}~,
\] 
and the expectation $\E_t$ in (\ref{e:expectedloss}) is w.r.t. the generation of labels $Y_t$, conditioned on both $\bx_t$, and
all previous $\bx$ and $Y$.

A key aspect of this formalization is that the Bayes optimal ordered subset 
\[
Y^*_t = {\rm argmin}_{Y = (j_1, j_2, \ldots, j_{|Y|}) \subseteq [K]} \E_t[\ell_{a,c}(Y_t,Y)]
\]
can be computed efficiently when knowing $\Delta_{1,t}, \dots, \Delta_{K,t}$. This is 
handled by the next lemma.  
In words, this lemma says that, in order to minimize (\ref{e:expectedloss}), 
it suffices to try out all possible sizes 
$s = 0, 1, \ldots, K$ for $Y^*_t$ and, for each such value, determine the sequence 
$Y^*_{s,t}$ that minimizes (\ref{e:expectedloss}) over all sequences of size $s$.
In turn, $Y^*_{s,t}$ can be computed just by sorting classes $i \in [K]$ in decreasing
order of $p_{i,t}$, sequence $Y^*_{s,t}$ being given by the first $s$ classes in this sorted 
list.
\begin{lemma}\label{l:bayes}
With the notation introduced so far, let $p_{i_1,t} \geq p_{i_2,t} \geq \ldots p_{i_K,t}$
be the sequence of $p_{i,t}$ sorted in nonincreasing order. Then we have that
\[
Y^*_t = {\rm argmin}_{s = 0, 1, \ldots K}  \E_t[\ell_{a,c}(Y_t,Y^*_{s,t})]~, 
\]
where $Y^*_{s,t} = (i_1,i_2, \ldots, i_s)$, and $Y^*_{0,t} = \emptyset$. 
\end{lemma}
\begin{proof}
First observe that, for any given size $s$, the sequence $Y^*_{s,t}$ must contain the $s$ 
top-ranked classes in the sorted order of $p_{i,t}$. This is because, for any candidate
sequence $Y_s =\{j_1, j_2, \ldots, j_s\}$, we have
$
\E_t[\ell_{a,c}(Y^*_t,Y_s)] 
= 
(1-a)\,\sum_{i \in Y_s} \left( c(j_i,s) - \left(\sfrac{a}{1-a} + c(j_i,s)\right)\,p_{i,t} \right)~.
$
If there exists $i \in  Y_s$ which is not among the $s$-top ranked ones, then we could replace
class $i$ in position $j_i$ within $Y_s$ with class $k \notin Y_s$ such that $p_{k,t} > p_{i,t}$
obtaining a smaller loss. 

Next, we show that the optimal ordering within  $Y^*_{s,t}$ is precisely ruled by the nonicreasing
order of $p_{i,t}$. 
By the sake of contradiction, assume there are $i$ and $k$ in $Y^*_{s,t}$ such that
$i$ preceeds $k$ in $Y^*_{s,t}$ but $p_{k,t} > p_{i,t}$. Specifically, let $i$ be in position
$j_1$ and $k$ be in position $j_2$ with $j_1 < j_2$ and such that $c(j_1,s) > c(j_2,s)$. 
Then, disregarding the common $(1-a)$-factor, switching the two classes within 
$Y^*_{s,t}$ yields an expected loss difference of
\begin{align*}
&c(j_1,s) - \left(\sfrac{a}{1-a} + c(j_1,s)\right)\,p_{i,t} 
+   
c(j_2,s) - \left(\sfrac{a}{1-a} + c(j_2,s)\right)\,p_{k,t}\\
&- \left(c(j_1,s) - \left(\sfrac{a}{1-a} + c(j_1,s)\right)\,p_{k,t} \right)
-   
\left(c(j_2,s) - \left(\sfrac{a}{1-a} + c(j_2,s)\right)\,p_{i,t}\right)\\
& = (p_{k,t} - p_{i,t})\,(c(j_1,s)-c(j_2,s)) > 0~,
\end{align*}
since $p_{k,t} > p_{i,t}$ and $c(j_1,s) > c(j_2,s)$. 
Hence switching would get a smaller loss which leads as a consequence to 
$Y^*_{s,t} = (i_1, i_2, \ldots, i_s)$. 
\end{proof}

%

Notice the way costs $c(i,s)$ 
influence the Bayes optimal computation. 
We see from (\ref{e:expectedloss}) that placing class $i$ within $\hY_t$ in position $j_i$
is beneficial (i.e., it leads to a reduction of loss) if and only if
\(
p_{i,t} >  c(j_i,|\hY_t|)/(\sfrac{a}{1-a} + c(j_i,|\hY_t|)).
\)
Hence, the higher is the slot $i_j$ in $\hY_t$ the larger should be $p_{i,t}$ in order
for this inclusion to be convenient.\footnote
{ 
Notice that this depends on the actual size of $\hY_t$,
so we cannot decompose this problem into $K$ independent problems. The decomposition
does occur if the costs $c(i,s)$ are constants, independent of 
$i$ and $s$, the criterion for inclusion
becoming $p_{i,t} \geq \theta$, for some constant threshold $\theta$. 
}

It is $Y^*_t$ above that we interpret as the true set of user preferences on $\bx_t$.
We would like to compete against $Y^*_t$ in a cumulative regret sense, i.e., 
we would like to bound
\[
R_T = \sum_{t=1}^T \E_t[\ell_{a,c}(Y_t,\hY_t)] - \E_t[\ell_{a,c}(Y_t,Y^*_t)]
\]
with high probability.
%

We use a similar but largely more general analysis than \cite{cg11}'s to 
devise an online second-order descent algorithm
whose updating rule makes the comparison vector 
$U = (\bu_1, \ldots, \bu_K) \in \R^{dK}$ defined through (\ref{e:labgenmult}) be Bayes optimal
w.r.t. a surrogate convex loss $L(\cdot)$ such that 
$g(\Delta) = - L'(\Delta)$.
Observe that the expected loss function 
defined in (\ref{e:expectedloss}) is, generally speaking, nonconvex in the margins $\Delta_{i,t}$ 
(consider, for instance the logistic case $g(\Delta) = \frac{1}{e^{\Delta}+1}$). 
Thus, we cannot directly minimize this expected loss.

\begin{figure}[!t!]
\begin{center}
\begin{small}
  \fbox{ \hspace*{1em}\begin{minipage}{0.9\textwidth}
      \mbox{}\\
{\bf Parameters:} loss parameters $a \in [0,1]$, cost values $c(i,s)$, 
interval $D = [-R,R]$, function $g\,:\,D \rightarrow \R$, confidence level
$\delta \in [0,1]$.\\
{\bf Initialization:} $A_{i,0} = I \in \R^{d\times d}$,\ $i = 1, \ldots, K$, \ 
$\bw_{i,1} = 0\in \R^{d}$, \ $i = 1, \ldots, K$;\\[2mm]
{\bf For} $t=1, 2\ldots, T:$
\begin{enumerate}
\nolineskips
\item Get instance $\bx_t \in \R^d\,:\,||\bx_t|| = 1$;
\item For $i \in [K]$, set $\hDelta'_{i,t} = \bx_t^{\top}{{\bw'_{i,t}}}$, where
\[
\bw'_{i,t} = 
\begin{cases}
\bw_{i,t} &{\mbox{if $\bw_{i,t}^{\top}\bx_t \in [-R,R]$}},\\ 
\bw_{i,t} - \left(\frac{\bw_{i,t}^{\top}\bx_t-R}{\bx_t^{\top}A_{i,t-1}^{-1}\bx_t}\right)\, A^{-1}_{i,t-1}\bx_t  
                                      &{\mbox{if $\bw_{i,t}^{\top}\bx_t > R$}},\\ 
\bw_{i,t} - \left(\frac{\bw_{i,t}^{\top}\bx_t+R}{\bx_t^{\top}A^{-1}_{i,t-1}\bx_t}\right)\, A^{-1}_{i,t-1}\bx_t  
                                      &{\mbox{if $\bw_{i,t}^{\top}\bx_t < -R $}};
\end{cases}
\]
\item Output
\[
\hY_t = {\rm argmin}_{Y = (j_1, j_2, ... j_{|Y|}) \subseteq [K]}  
\left( \sum_{i \in Y} \left( c(j_i,|Y|) - \left(\sfrac{a}{1-a} + c(j_i,|Y|)\right)\,\hp_{i,t} \right)\right)~,
\]
where
\[
\hp_{i,t} = p([\hDelta'_{i,t}+\epsilon_{i,t}]_D) 
          = \frac{g(-[\hDelta'_{i,t}+\epsilon_{i,t}]_D)}
                                   {g([\hDelta'_{i,t}+\epsilon_{i,t}]_D)+g(-[\hDelta'_{i,t}+\epsilon_{i,t}]_D)},
\]
and
\[
\epsilon^2_{i,t} 
= 
\bx_t^\top A^{-1}_{i,t-1} \bx_t
\left(U^2 +\frac{d\,c'_L}{(c''_L)^2}\ln \left(1+ \frac{t-1}{d} \right)
+ \frac{12}{c''_L}\left(\frac{c'_L}{c''_L} + 3 L(-R)\right)\ln \frac{K(t+4)}{\delta} \right);
\]
\item Get feedback $Y_t\cap \hY_t$;
\item For $i \in [K]$, update: 
\[
A_{i,t} = A_{i,t-1} + |s_{i,t}| \bx_t \bx_t^\top,
\qquad 
\bw_{i,t+1} = \bw'_{i,t} -\frac{1}{c''_L}A^{-1}_{i,t} \nabla_{i,t},
\] 
where
\[
s_{i,t} =
\begin{cases}
1 & {\mbox{If $i \in Y_t\cap \hY_t$}}\\
-1 & {\mbox{If $i \in \hY_t\setminus Y_t = \hY_t \setminus (Y_t\cap \hY_t)$}}\\
0 & {\mbox{otherwise}};
\end{cases}
\]
%
and
\[
\nabla_{i,t} = \nabla_{\bw} L(s_{i,t}\,\bw^\top \bx_t)|_{\bw = \bw'_{i,t}} 
             = - g(s_{i,t}\,\hDelta'_{i,t})\,s_{i,t}\,\bx_t.
\]
\end{enumerate}
\mbox{}
\end{minipage}
\hspace*{1em}
}
\end{small}
\end{center}
\caption{\label{f:2}The partial feedback algorithm in the (ordered) multiple label setting.}
\end{figure}

\section{Algorithm and regret bounds}\label{s:alg}
In Figure \ref{f:2} is our bandit algorithm for (ordered) multiple labels.
 The algorithm is based on
replacing the unknown model vectors $\bu_1, \ldots,\bu_K$ with prototype vectors 
$\bw'_{1,t}, \ldots,\bw'_{K,t}$, being $\bw'_{i,t}$ the time-$t$ approximation
to $\bu_i$, satisying similar constraints
we set for the $\bu_i$ vectors. For the sake of brevity, we let 
$\hDelta'_{i,t} = \bx_t^\top\bw'_{i,t}$, 
and $\Delta_{i,t} = \bu_{i}^\top\bx_t$, $i \in [K]$.

The algorithm uses $\hDelta'_{i,t}$ as proxies for the underlying $\Delta_{i,t}$
according to the (upper confidence) approximation scheme 
$\Delta_{i,t} \approx [\hDelta'_{i,t} + \epsilon_{i,t}]_D$,
%
%
%
where $\epsilon_{i,t} \geq 0$ is a suitable upper-confidence level for class $i$
at time $t$, and $[\cdot]_D$ denotes the clipping-to-$D$ operation:
If $D = [-R,R]$,
then 
%
\[
[x]_D =
\begin{cases}
R &{\mbox{if $x > R$}}\\ 
x &{\mbox{if $ -R \leq x \leq R$}}\\ 
-R &{\mbox{if $x < -R$}}~.
\end{cases}
\] 
The algorithm's prediction at time $t$ has the same form as the computation of the Bayes optimal 
sequence $Y^*_t$, where we replace the true (and unknown) 
$p_{i,t} = p(\Delta_{i,t})$ with the corresponding upper confidence proxy 
\[
\hp_{i,t} = p([\hDelta'_{i,t}+\epsilon_{i,t}]_D)~.
\]
being
\begin{align*}
\hY_t = {\rm argmin}_{Y = (j_1, j_2, ... j_{|Y|}) \subseteq [K]}  
\left( \sum_{i \in Y} \left( c(j_i,|Y|) - \left(\sfrac{a}{1-a} + c(j_i,|Y|)\right)\,\hp_{i,t} \right)\right)~.
\end{align*}
Computing $\hY_t$ above can be done by mimicking the computation of the Bayes optimal 
ordered subset 
$Y_t^*$ (just replace $p_{i,t}$ by $\hp_{i,t}$).
From a computational viewpoint, this essentially amounts to sorting 
classes $i \in [K]$ in decreasing value of $\hp_{i,t}$, 
i.e., order of $K \log K$ running time per prediction.
Thus the algorithm is producing a ranked list of relevant classes based on 
upper-confidence-corrected scores $\hp_{i,t}$. Class $i$ is deemed relevant
and ranked high among the relevant ones when either $\hDelta'_{i,t}$ is a good approximation to 
$\Delta_{i,t}$ and $p_{i,t}$ is large, or when the algorithm is not very confident on its own
approximation about $i$ (that is, the upper confidence level $\epsilon_{i,t}$ is large).
%
%


The algorithm in Figure \ref{f:2} receives in input the loss parameters $a$ and $c(i,s)$, 
the model function $g(\cdot)$ and the associated margin domain $D = [-R,R]$, and maintains both $K$ 
positive definite matrices $A_{i,t}$ of dimension $d$ (initially set to
the $d\times d$ identity matrix), and $K$ weight vector $\bw_{i,t} \in \R^d$ 
(initially set to the zero vector).
At each time step $t$, upon receiving the $d$-dimensional instance vector $\bx_t$
the algorithm uses the weight vectors $\bw_{i,t}$ to compute the 
prediction vectors $\bw'_{i,t}$. These vectors can easily be seen as the result
of projecting $\bw_{i,t}$ onto interval $D = [-R,R]$ w.r.t. the
distance function $d_{i,t-1}$, i.e., 
\[
\bw'_{i,t} = {\rm argmin}_{\bw \in \R^d\,:\, \bw^\top\bx_t \in D}\, 
d_{i,t-1}(\bw,\bw_{i,t}), i \in [K],
\]
where 
\[
d_{i,t}(\bu,\bw) = (\bu-\bw)^{\top}\,A_{i,t}\,(\bu-\bw)~.
\] 
Vectors $\bw'_{i,t}$ are then used to produce prediction values $\hDelta'_{i,t}$ 
involved in the upper-confidence calculation of the predicted ordered subset 
$\hY_t \subseteq [K]$. 
Next, the feedback  $Y_t\cap \hY_t$ is observed, and the algorithm
in Figure \ref{f:2} promotes all classes $i \in Y_t\cap \hY_t$ (sign $s_{i,t} = 1$),
demotes all classes $i \in \hY_t\setminus Y_t$ (sign $s_{i,t} = -1$), and leaves
all remaining classes $i \notin \hY_t$ unchanged (sign $s_{i,t} = 0$).
Promotion of class $i$ on $\bx_t$ implies that if the new vector $\bx_{t+1}$ is close to $\bx_t$ 
then $i$ will be ranked higher on $\bx_{t+1}$.
The update $\bw'_{i,t} \rightarrow \bw_{i,t+1}$ is 
based on the gradients $\nabla_{i,t}$ of a loss function 
$L(\cdot)$ satisfying $L'(\Delta) = -g(\Delta)$. 
On the other hand, the update $A_{i,t-1} \rightarrow A_{i,t}$ 
uses the rank one matrix\footnote
{
The rank-one update is based on $\bx_t\bx_t^{\top}$
rather than $\nabla_{i,t}\nabla_{i,t}^{\top}$,
as in, e.g., \cite{hka07}. This is due to technical reasons that will be made clear 
in Section \ref{s:tech}. This feature tells this algorithm slightly 
apart from the Online Newton step algorithm \cite{hka07}, which is the starting
point of our analysis.
}
$\bx_t\bx_t^{\top}$.
In both the update of $\bw'_{i,t}$ and the one involving $A_{i,t-1}$, the reader 
should observe the role played by the signs $s_{i,t}$.
Finally, the constants $c'_L$ and $c''_L$ occurring in the expression for $\epsilon^2_{i,t}$
are related to smoothness properties of $L(\cdot)$, as explained in the next theorem.\footnote
{
The proof is given in Section \ref{s:tech}.
}

%

\begin{theorem}\label{t:cumregret}
Let $L\,:\, D = [-R,R] \subseteq \R \rightarrow \R^+$ be a $C^2(D)$ convex and nonincreasing function of its 
argument,
$(\bu_1,\ldots,\bu_K) \in \R^{dK}$ be defined in (\ref{e:labgenmult}) with $g(\Delta) = - L'(\Delta)$ for all 
$\Delta \in D$, and such that $\|\bu_i\| \leq U$ for all $i \in [K]$. 
Assume there are positive constants $c_L$, $c'_L$ and $c''_L$ such that:
\begin{enumerate}
\item[i.]
$\frac{L'(\Delta)\,L''(-\Delta) + L''(\Delta)\,L'(-\Delta)}{(L'(\Delta) + L'(-\Delta))^2} \geq -c_L$,
\item[ii.] $(L'(\Delta))^2 \leq c'_L$, 
\item [iii.] $L''(\Delta) \geq c''_L$
\end{enumerate}
simultaneously hold for all $\Delta \in D$. Then the cumulative regret $R_T$
of the algorithm in Figure \ref{f:2} satisfies, with probability at least $1-\delta$,
\[
R_T = O \left((1-a)\,c_L\,K\,\sqrt{T\,C\,d\,\ln \left(1+ \frac{T}{d} \right)} \right),\notag
\]
where 
\[
C = O\left(U^2+\frac{d\,c'_L}{(c''_L)^2}\,\ln\left(1+\frac{T}{d}\right) 
+ \left(\frac{c'_L}{(c''_L)^2} + \frac{L(-R)}{c''_L}\right)\,\ln \frac{KT}{\delta}\right).
\]
\end{theorem}
It is easy to see that when $L(\cdot)$ is the square loss 
$L(\Delta) = (1-\Delta)^2/2$ 
and $D = [-1,1]$, we have $c_L = 1/2$, $c'_L = 4$ and $c''_L = 1$;
when $L(\cdot)$ is the logistic loss $L(\Delta) = \ln(1+e^{-\Delta})$ 
and $D = [-R,R]$, we have $c_L = 1/4$, $c'_L \leq 1$ and 
$c''_L = \frac{1}{2(1+\cosh(R))}$, 
where $\cosh(x) = \frac{e^x+e^{-x}}{2}$.

The following remarks are in order at this point.
\begin{remark}
A drawback of Theorem \ref{t:cumregret} is that, in order to properly set 
the upper confidence levels $\epsilon_{i,t}$,
we assume prior knowledge of the norm upper bound $U$. 
Because this information is often unavailable, we present here a simple 
modification to the algorithm 
that copes with this limitation. We change the definition of $\epsilon^2_{i,t}$ 
in Figure \ref{f:2} to
%
%
\[
\epsilon^2_{i,t}= \max\Biggl\{\bx^\top A^{-1}_{i,t-1} \bx\,
\left(\frac{2\,d\,c'_L}{(c''_L)^2}\,\ln \left(1+ \frac{t-1}{d} \right)
+ \frac{12}{c''_L}\,\left(\frac{c'_L}{c''_L} + 3 L(-R)\right)\,\ln \frac{K(t+4)}{\delta} \right), 4\,R^2\Biggl\}\,.
\]
that is, we substitute $U^2$ by $\frac{d\,c'_L}{(c''_L)^2}\,\ln \left(1+ \frac{t-1}{d} \right)$, 
and cap the maximal value of 
$\epsilon^2_{i,t}$ to $4\,R^2$. 
This immediately leads to the following result.\footnote
{
The proof is deferred to Section \ref{s:tech}.
}
\begin{theorem}\label{t:cumregret_logt}
With the same assumptions and notation as in Theorem \ref{t:cumregret}, if we replace $\epsilon^2_{i,t}$
as explained above we have that, with probability at least $1-\delta$, $R_T$ satisfies 
\[
R_T = O \left((1-a)\,c_L\,K\,\sqrt{T\,C\,d\,\ln \left(1+ \frac{T}{d} \right)} 
+ (1-a)\,c_L\,K\,R\,d\,\left(\exp\left(\frac{(c''_L)^2\,U^2}{c'_L\,d}\right)-1 \right) \right)~.
\]
\end{theorem}
\end{remark}

\begin{remark}\label{r:comp}
From a computational standpoint, the most demanding operation in Figure \ref{f:2} 
is computing the upper confidence
levels $\epsilon_{i,t}$ involving the inverse matrices $A^{-1}_{i, t-1}$, $i \in [K]$. This
can be done incrementally in $\mathcal{O}(K\,d^2)$ time per round, which makes it hardly practical
if both $d$ and $K$ are large. In practice (as explained, e.g., by \cite{cg11}), one can use
a version of the algorithm which maintains {\em diagonal} matrices $A_{i,t}$ instead of full ones.
All the steps remain the same except Step $5$ of Algorithm~\ref{f:2} where one
defines the $r$th diagonal element of matrix $A_{i,t}$ as 
$(A_{i,t})_{r,r}= (A_{i,t-1})_{r,r} + x_{r,t}^2$, being 
$\bx_t = (x_{1,t}, x_{2,t}, \ldots, x_{r,t}, \ldots, x_{K,t})^\top$.
The resulting running time per round (including prediction and update) becomes
$\mathcal{O}(dK + K \log K)$.
In fact, when a limitation on the size of $\hY_t$ is given, the running time may be further 
reduced, see Remark \ref{r:typical}.
\end{remark}

\section{On ranking with partial feedback}\label{s:rank}
As Lemma \ref{l:bayes} points out, when the cost values $c(i,s)$ in 
the loss function 
$\ell_{a,c}$ are {\em stricly} decreasing 
i.e., $c(1,s) > c(2,s) > \ldots > c(s,s)$, for  all $s \in [K]$, 
then the Bayes
optimal ordered sequence $Y^*_t$ on $\bx_t$ is unique 
can be obtained by sorting classes in decreasing values of 
$p_{i,t}$, and then decide on a cutoff point\footnote
{
This is called the {\em zero point} by \cite{fhlmb08}.
}
induced by the loss parameters, so as to tell relevant classes apart from irrelevant ones. 
In turn, because $p(\Delta) = \frac{g(-\Delta)}{g(\Delta)+g(-\Delta)}$ 
is increasing in $\Delta$, this ordering corresponds to sorting classes in decreasing values of 
$\Delta_{i,t}$.
Now, if parameter $a$ in $\ell_{a,c}$ is very close\footnote
{
If $a=1$, the algorithm only cares about false negative 
mistakes, the best strategy being always predicting $\hY_t = [K]$.
Unsurprisingly, this yields zero regret in both Theorems \ref{t:cumregret} and \ref{t:cumregret_logt}.
} 
to $1$, then $|Y^*_t| = K$, and the algorithm itself
will produce ordered subsets $\hY_t$ such that $|\hY_t|=K$. 
Moreover, it does so by receiving {\em full} feedback on the relevant classes 
at time $t$ (since $Y_t \cap \hY_t = Y_t$).
As is customary (e.g., \cite{dwch12}), one can view any multilabel assignment 
$Y = (y_1,\ldots,y_K) \in \{0,1\}^K$ as a ranking among
the $K$ classes in the most natural way: $i$ preceeds $j$ if and only if $y_i > y_j$.
The (unnormalized) ranking loss function $\ell_{rank}(Y,f)$ between the multilabel 
$Y$ and a ranking function $f\,:\, \R^d \rightarrow \R^K$,
representing degrees of class relevance sorted in a decreasing order 
$f_{j_1}(\bx_t) \geq f_{j_2}(\bx_t) \geq \ldots \geq f_{j_K}(\bx_t) \geq 0$, 
counts the number of class pairs that disagree in the two rankings:
\[
\ell_{rank}(Y,f) = \sum_{i,j \in [K]\,:\, y_i > y_j} \left( \{ f_{i}(\bx_t) < f_{j}(\bx_t) \} 
+ \sfrac{1}{2}\,\{ f_{i}(\bx_t) = f_{j}(\bx_t) \} \right),
\]
where $\{\ldots\}$ is the indicator function of the predicate at argument.
As pointed out by \cite{dwch12}, the ranking function $f(\bx_t) = (p_{1,t}, \ldots, p_{K,t})$
is also Bayes optimal w.r.t. $\ell_{rank}(Y,f)$, {\em no matter if} the class labels $y_i$ are
conditionally independent or not. Hence we can use the algorithm in Figure \ref{f:2} with
$a$ close to $1$ for tackling ranking problems 
derived from multilabel ones, when the measure of choice is $\ell_{rank}$ and the feedback is full.

We now consider a partial information version of the above ranking problem. 
Suppose that at each time $t$, the environment discloses both $\bx_t$ and a maximal {\em size} $S_t$ for
the ordered subset $\hY_t = (j_1,j_2, \ldots, j_{|\hY_t|})$  
(both $\bx_t$ and $S_t$ can be chosen adaptively by an adversary).
Here $S_t$ might be the number of available slots in a webpage
or the number of URLs returned by a search engine in response to query 
$\bx_t$. 
Then it is plausible to compete in a regret sense against the best time-$t$ offline ranking of the form 
\[
f^*(\bx_t) = f^*(\bx_t; S_t) = (f^*_{1}(\bx_t), f^*_{2}(\bx_t), \ldots, f^*_{K}(\bx_t)), 
\]
where the number of strictly positive $f^*_i(\bx_t)$ values is at most $S_t$.
Further, the ranking loss could be reasonably restricted to count the number of class pairs disagreeing
within $\hY_t$ plus a quantity related to the number of false negative mistakes. If $\hY_t$ is the
sequence of length $S_t$ associated with ranking function $f$, we consider the loss function 
$\ell_{p-rank,t}$ (``partial information $\ell_{rank}$ at time $t$")
%
\[
\ell_{p-rank,t}(Y,f) = \sum_{i,j \in \hY_t\,:\, y_i > y_j} \left( \{ f_{i}(\bx_t) < f_{j}(\bx_t) \} 
+ \sfrac{1}{2}\,\{ f_{i}(\bx_t) = f_{j}(\bx_t) \} \right) + S_t\,|Y_t\setminus \hY_t|~.
\]
In this loss function, the factor $S_t$ multiplying $|Y_t\setminus \hY_t|$ 
serves as balancing the contribution of the double sum $\sum_{i,j \in \hY_t\,:\, y_i > y_j}$
with the contribution of false negative mistakes $|Y_t\setminus \hY_t|$. 
For convenience, we will interchangeably
use the notations $\ell_{p-rank,t}(Y,f)$ and $\ell_{p-rank,t}(Y,\hY_t)$, whenever it
is clear from the surrounding context that $\hY_t$ is the sequence corresponding to $f$.

The next lemma\footnote
{
We postpone the lengthy proof to Section \ref{s:tech}.
}
is the ranking counterpart to Lemma \ref{l:bayes}. It shows that
the Bayes optimal ranking for $\ell_{p-rank,t}$ is given by
\[
f^*(\bx_t; S_t) = (p'_{1,t}, p'_{2,t}, \ldots, p'_{K,t}),
\]
where $p'_{j,t} = p_{j,t}$ if $p_{j,t}$ is among the $S_t$ largest values in the sequence 
$(p_{1,t}, \ldots, p_{K,t})$, and 0 otherwise.
That is, $f^*(\bx_t; S_t)$ is the function that ranks classes according to decreasing values of $p_{i,t}$
and cuts off exactly at position $S_t$. In order for this result to go through
we need to restrict model (\ref{e:labgenmult}) to the case of conditionally
independent classes, i.e., to the case when
\begin{equation}\label{e:indep}
\Pr_t(y_{1,t}, \ldots, y_{K,t}) = \prod_{i\in[K]} p_{i,t}\,.
\end{equation}
This is in striking contrast to the full information setting, where
the Bayes optimal ranking only depends on the marginal distribution values 
$p_{i,t}$ \cite{dwch12}. Due to the interaction between the two terms in the
definition of $\ell_{p-rank,t}$, the Bayes optimal ranking for $\ell_{p-rank,t}$ 
turns out to depend
on both marginal and pairwise correlation values of the joint class distribution.
This would force us to maintain $O(K^2)$ upper confidence values $\epsilon_{i,j}$,
one for each pair $(i,j), i < j$, leading to an extra computational burder which can
also become prohibitive when the number of classes $K$ is large.

\begin{lemma}\label{l:bayesrank}
With the notation introduced so far, let the joint distribution
$\Pr_t(y_{1,t}, \ldots, y_{K,t})$ factorize as in (\ref{e:indep}).
Then $f^*(\bx_t; S_t)$ introduced above satisfies
\[
f^*(\bx_t; S_t) = {\rm argmin}_{Y = (i_1, i_2, ... i_{h})\,, h \leq S_t}  \E_t[\ell_{p-rank,t}(Y_t,Y)]~.
\]
\end{lemma}
If we add to the argmin of our algorithm (Step 3 in Figure \ref{f:2}) the further constraint $|Y| \leq S_t$
(notice that the resulting computation is still about sorting classes according to decreasing values 
of $\hp_{i,t}$), we are defining a partial information ranking algorithm that ranks classes according
to decreasing values of $\hp_{i,t}$ up to position $S_t$ (i.e., $|\hY_t| = S_t$). Let $\hf(\bx_t,S_t)$
be the resulting ranking. We can then define the cumulative regret $R_{T}$ w.r.t. $\ell_{p-rank,t}$ as
\begin{equation}\label{e:rankingregret}
R_{T} = \sum_{t=1}^T \E_t[\ell_{p-rank,t}(Y_t,\hf(\bx_t,S_t))] 
- \E_t[\ell_{p-rank,t}(Y_t,f^*(\bx_t,S_t)], 
\end{equation}
that is, the amount to which the conditional $\ell_{p-rank,t}$-risk of $\hf(\bx_t,S_t)$
exceeds the one of the Bayes optimal ranking $f^*(\bx_t; S_t)$, cumulated over time.

We have the following ranking counterpart to Theorem \ref{t:cumregret}.
\begin{theorem}\label{t:cumregret_rank}
With the same assumptions and notation as in Theorem \ref{t:cumregret}, combined
with the independence assumption (\ref{e:indep}), 
let the cumulative regret $R_{T}$ w.r.t. $\ell_{p-rank,t}$ be defined 
as in (\ref{e:rankingregret}).
%
Then, with probability at least $1-\delta$, we have that the algorithm in Figure \ref{f:2}
working with $a \rightarrow 1$ and strictly decreasing cost values $c(i,s)$
(i.e., the one computing in round $t$ the ranking function 
$\hf(\bx_t,S_t)$) achieves
\[
R_{T} = O\left(c_L\,\sqrt{S\,K\,T\,C\,d\,\ln \left(1+ \frac{T}{d} \right)} \right),
\]
where $S = \max_{t = 1,\ldots, T}S_t$.
\end{theorem}
The proof (see Section \ref{s:tech}) is very similar to the one of Theorem \ref{t:cumregret}. 
This suggests that, to some extent, we are decoupling the label 
generating model from the loss function $\ell$ under consideration.

\begin{remark}\label{r:typical}
As is typical in many multilabel classification settings, the number of classes $K$ can either
be very large or have an inner structure (e.g., a hierarchical or DAG-like structure). 
It is often the case that in such a large label space, many classes are relatively rare. 
This has lead researchers to consider methods that are specifically taylored to leverage the 
label sparsity of the chosen classifier (e.g., \cite{hklz09} and references therein) 
and/or the specific structure
of the set of labels (e.g., \cite{cgz06,bk11}, and references therein). Though our algorithm
is not designed to exploit the label structure, we would like to stress that 
the restriction $|\hY_t| \leq S_t \leq S$ in Theorem \ref{t:cumregret_rank}
allows us to replace the linear dependence on the total number of classes $K$ 
(which is often much larger than $S$) by $\sqrt{SK}$. It is very easy to see that 
this restriction would bring similar benefits to Theorem \ref{t:cumregret}. 

The above restriction is not only beneficial from a ``statistical"
point of view, but also from a computational one. In fact, as is by now standard, algorithms like the
one in Figure \ref{f:2} can easily be cast in dual variables (i.e., in a RKHS). This comes with 
at least two consequences:
\begin{enumerate}
\item We can depart from the (generalized) linear modeling assumption (\ref{e:labgenmult}),
and allow for more general nonlinear dependences of $p_{i,t}$ on the input vectors $\bx_t$.
\item We can maintain a dual variable representation for margins $\hDelta'_{i,t}$
and quadratic forms $\bx_t^\top A^{-1}_{i,t-1}\bx_t$, so that computing each one of them
takes
$O(N^2_{i,t-1})$ inner products, where $N_{i,t}$ is the number of times class $i$ has been updated 
up to time $t$, each inner product being $O(d)$. 
Now, each of the (at most $S_t \leq S$) updates 
is $O(N^2_{i,t-1})$. Hence, the overall running time in round $t$ is coarsely overapproximated by 
$O(d\,\sum_{i \in [K]} N^2_{i,T} + K\log K)$. From $\sum_{i \in [K]} N_{i,T} \leq ST$, we see that
when $S$ is small compared to $K$, then $N_{i,t-1}$ tends to be small as well.
For instance, if $S \leq \sqrt{K}$ this leads to a running time per round of the form 
$SdT^2$, which can be smaller than $Kd^2$ mentioned in Remark \ref{r:comp}.
\end{enumerate}
Finally, observe that one can also combine Theorem \ref{t:cumregret_rank} with the 
argument contained in Remark 1.
\end{remark}

\section{Experiments}\label{s:exp}
The experiments we report here are meant to validate the 
exploration-exploitation tradeoff implemented by our algorithm under
different conditions (restricted vs. nonrestricted number of classes),
loss measures ($\ell_{a,c}$, $\ell_{rank,t}$, and Hamming loss)
and model/parameter settings ($L$ = square loss, $L$ = logistic loss, with varying $R$).

\noindent{\bf Datasets. }
We used two multilabel datasets. The first one, called Mediamill, was introduced in a video 
annotation challenge~\cite{SnoekWGGS06}. It comprises 30,993 training samples and 12,914 test ones. The number of features $d$ is 120, and the number of classes $K$ is 101.
The second dataset is Sony CSL Paris~\cite{PachetR2009}, made up of 16,452 train 
samples and 16,519 test samples, each sample being described by $d = 98$ features.
The number of classes $K$ is 632. 
In both cases, feature vectors have been normalized to unit L2 norm.

\noindent{\bf Parameter setting and loss measures. }
We used the algorithm in Figure \ref{f:2} with
two different loss functions, the square loss and the logistic loss, 
and varied the parameter $R$ for the latter. 
The setting of the cost function $c(i,s)$ depends on the task at hand, 
and for this preliminary experiments we decided to evaluate two possible settings
only. 
The first one, denoted by ``decreasing $c$'' is $c(i,s)  = \frac{s-i+1}{s}, 
i = 1, \ldots, s$, the second one, denoted by ``constant $c$'', is 
$c(i,s)  = 1,$ for all $i$ and $s$. In all experiments, the $a$ parameter was set 
to 0.5, so that $\ell_{a,c}$ with constant $c$ reduces to half the Hamming loss. 
In the decreasing $c$ scenario, we evaluated the performance of the algorithm 
on the loss $\ell_{a,c}$ that the algorithm is minimizing, but also
its ability to produce meaningful (partial) rankings through $\ell_{rank,t}$.
On the constant $c$ setting, we evaluated the Hamming loss.
As is typical of multilabel problems, the label {\em density}, i.e., the average
fraction of labels associated with the examples, is quite small. For instance, on 
Mediamill this is 4.3\%. Hence, it is clearly beneficial to impose an upper bound
$S$ on $|\hY_t|$. For the constant $c$ and ranking loss experiments we tried out different values of $S$,
and reported the final performance.

\noindent{\bf Baseline. }
As baseline, we considered a full information version of Algorithm~\ref{f:2} using the square loss, 
that receives after each prediction the full array of true labels $Y_t$
for each sample. We call this algorithm OBR (Online Binary Relevance),
because it is a natural online adaptation of the binary relevance algorithm, 
widely used as a baseline in the multilabel literature. Comparing to OBR stresses the effectiveness of the exploration/exploitation
rule above and beyond the details of underlying generalized linear predictor. 
OBR was used to produce subsets (as in the Hamming loss case), 
and restricted rankings (as in the case of $\ell_{rank,t}$).

\noindent{\bf Results. }
Our results are summarized in Figures~\ref{fig:sony} and~\ref{fig:mediamill}.
The algorithms have been trained by sweeping only once over the training data. 
Though preliminary in nature, these experiments allow us to draw a few
conclusions.
Our results for the avarage $\ell_{a,c}(Y_t,\hY_t)$ with decreasing $c$ are 
contained in the two left plots.
We can see that the performance is improving over time on both datasets, as predicted by Theorem~\ref{t:cumregret}.
In the middle plots are the final cumulative Hamming losses with constant $c$ 
divided by the number of training samples, as a function of $S$. Similar plots
are on the right with the final average ranking losses $\ell_{rank,t}$ divided by $S$.
In both cases we see that there is an optimal value of $S$ that allows to balance the exploration and the exploitation of the algorithm. Moreover the performance of our algorithm is always pretty close to the performance of OBR, even if our algorithm is receiving only partial feedback. In many experiments the square loss seems 
to give better results. Exception is the ranking loss on the Mediamill dataset
(Figure \ref{fig:mediamill}, right).
%
%
\begin{figure}[t]
  \centering \footnotesize
  \begin{tabular}{c@{\hspace{0.01\linewidth}}c@{\hspace{0.01\linewidth}}c@{}}
  \includegraphics[width=0.32\linewidth]{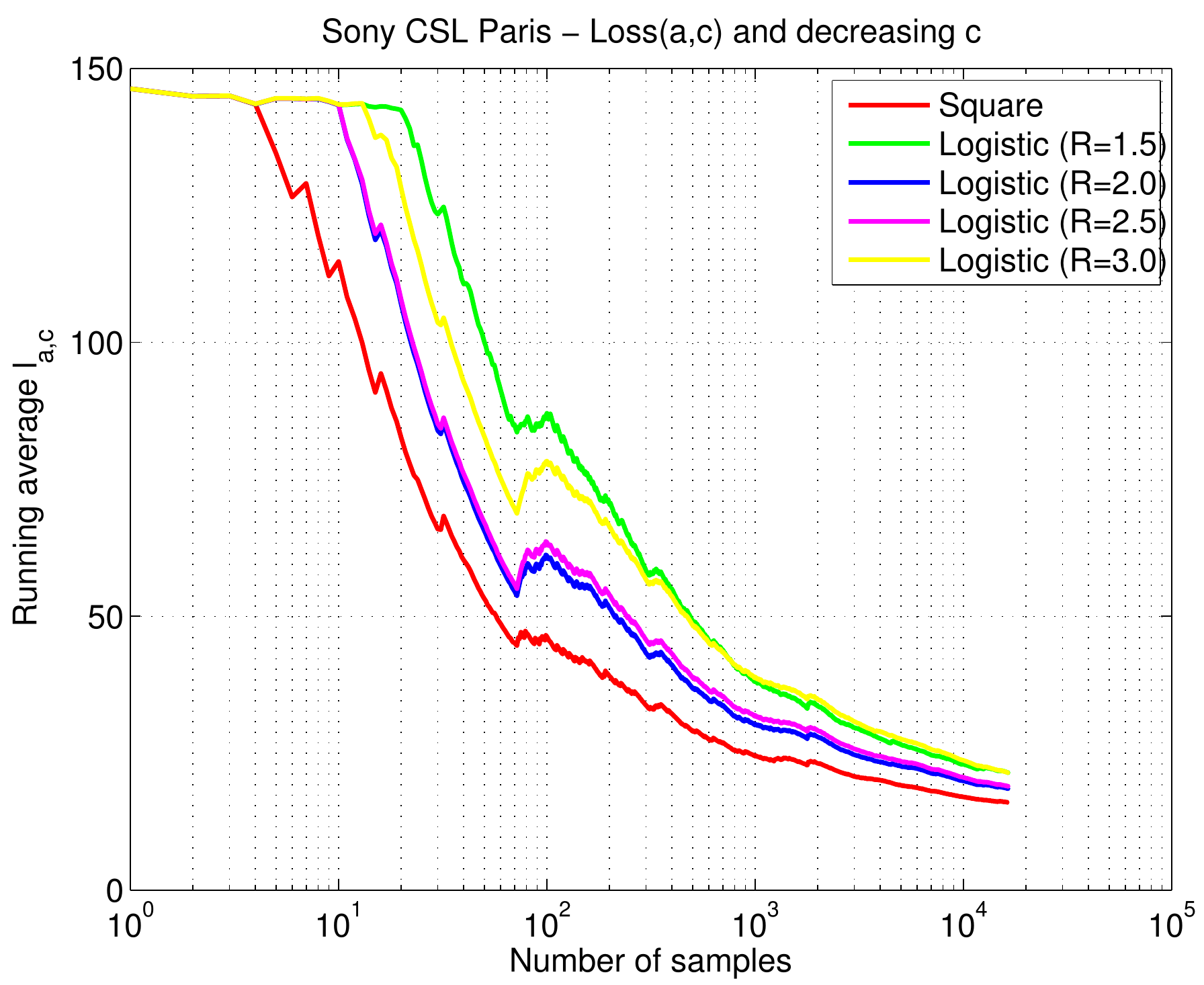} &
  \includegraphics[width=0.32\linewidth]{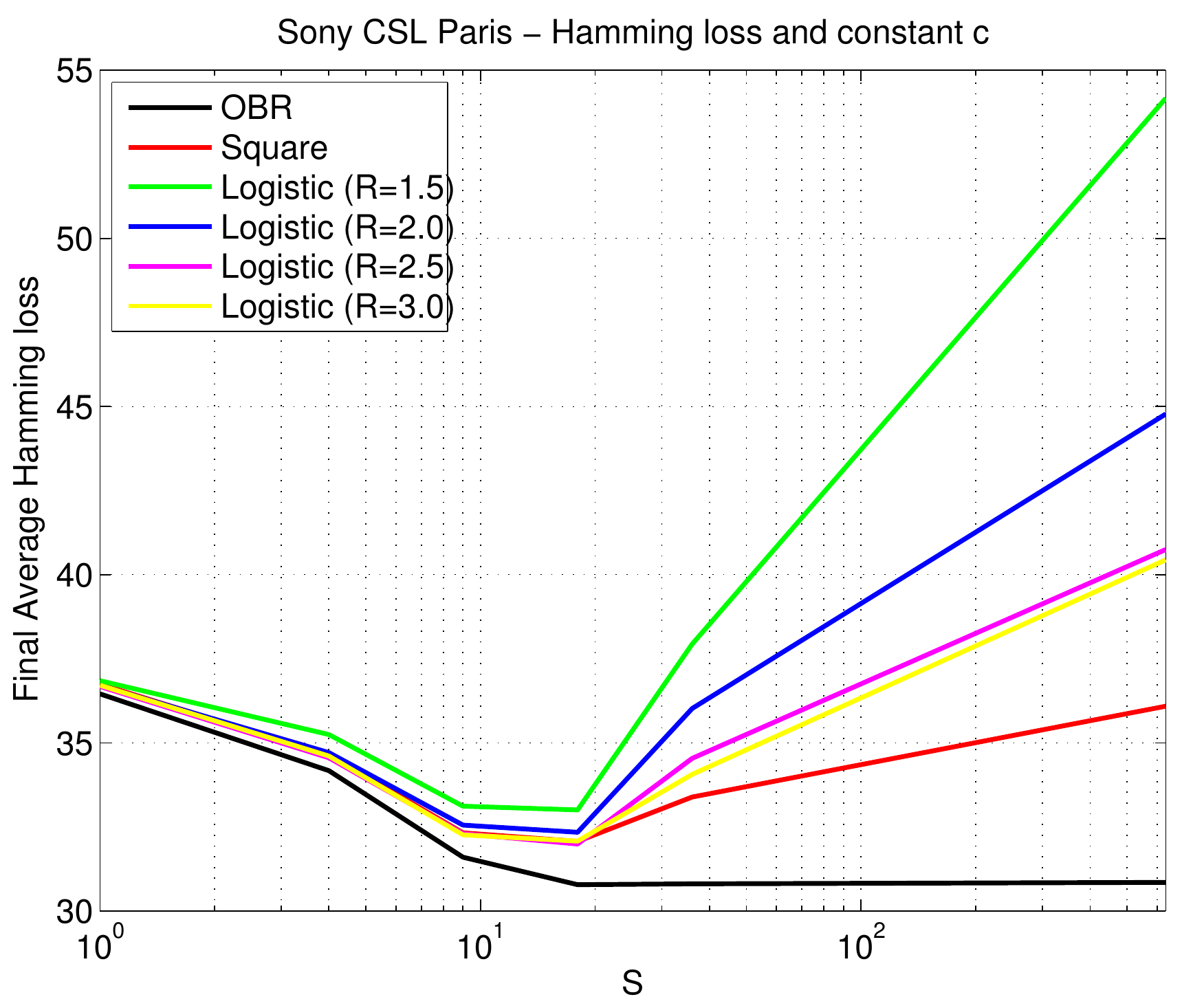} &
  \includegraphics[width=0.32\linewidth]{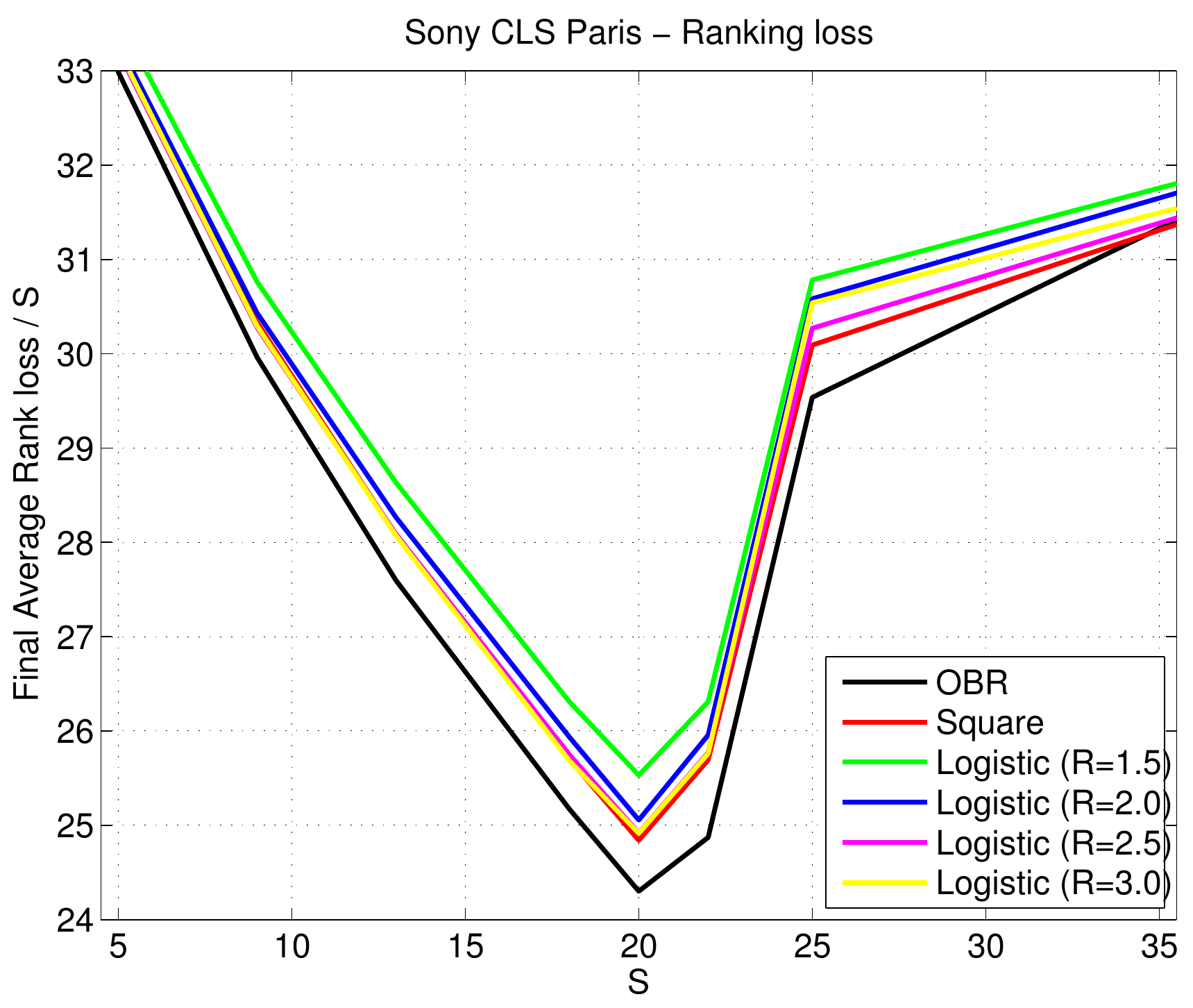}
  \end{tabular}
\caption{Experiments on the Sony CSL Paris dataset.\vspace{-0.2in}}
\label{fig:sony}
\end{figure}
\begin{figure}[t]
  \centering \footnotesize
  \begin{tabular}{c@{\hspace{0.01\linewidth}}c@{\hspace{0.01\linewidth}}c@{}}
  \includegraphics[width=0.32\linewidth]{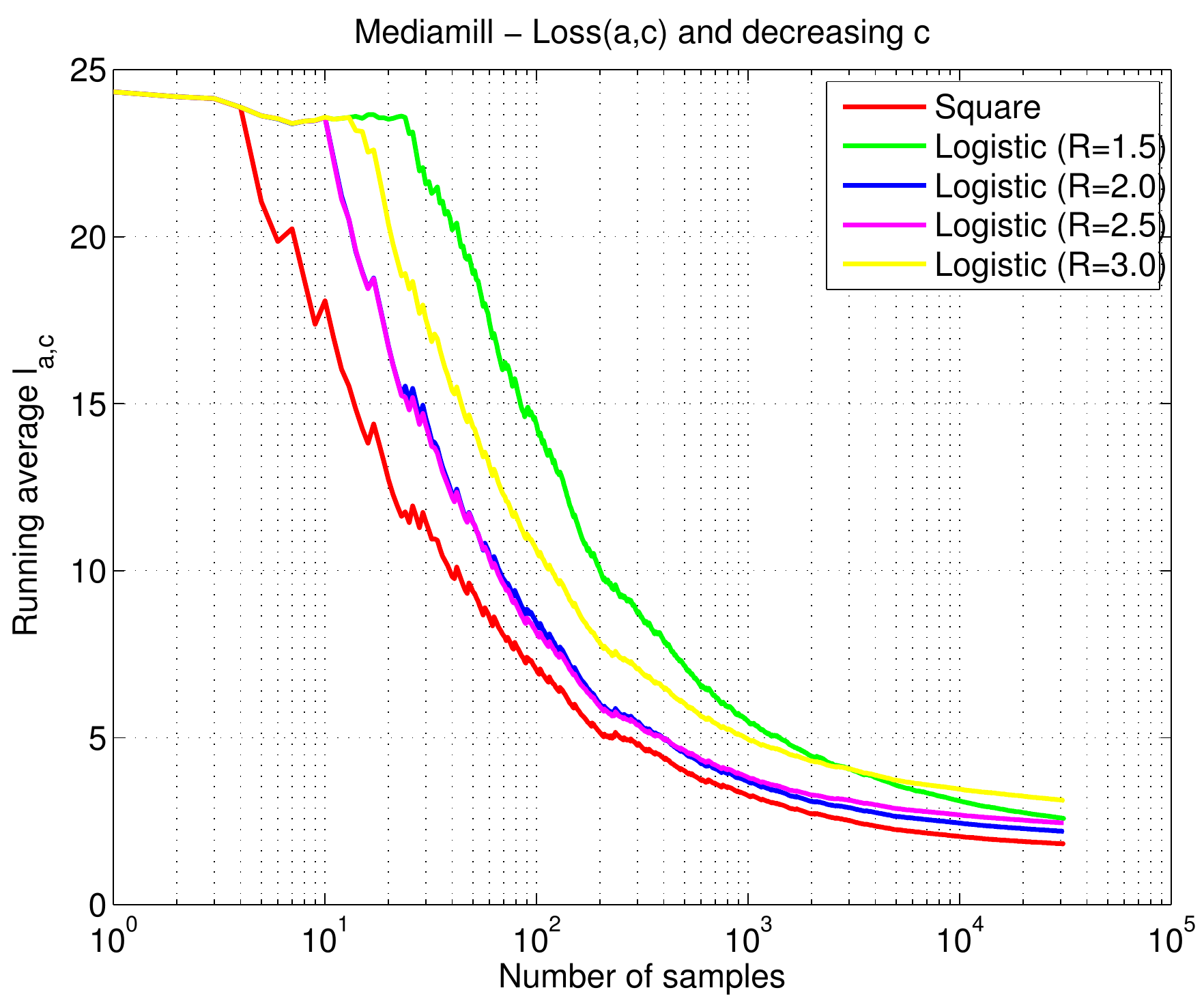} &
  \includegraphics[width=0.32\linewidth]{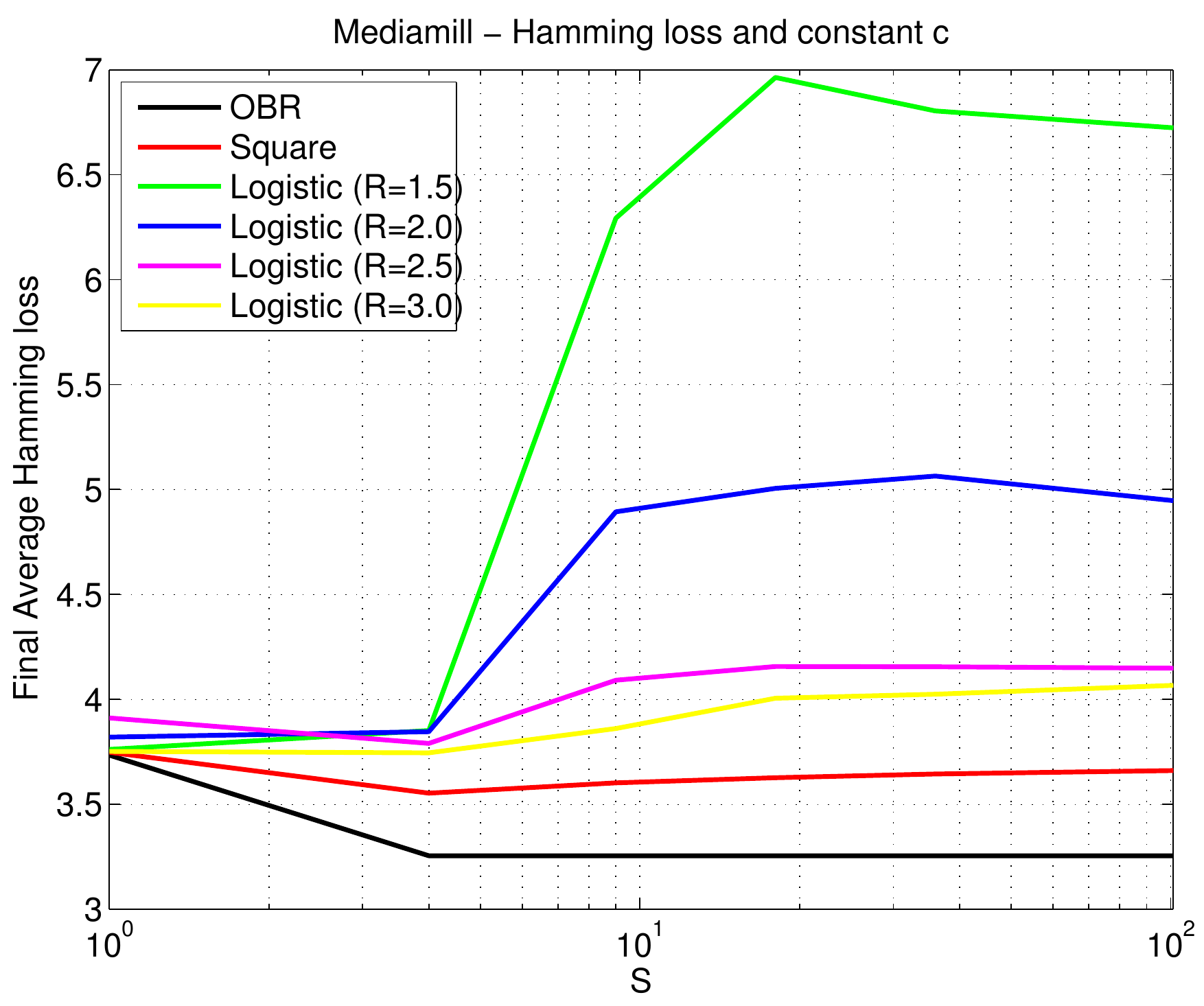} &
  \includegraphics[width=0.32\linewidth]{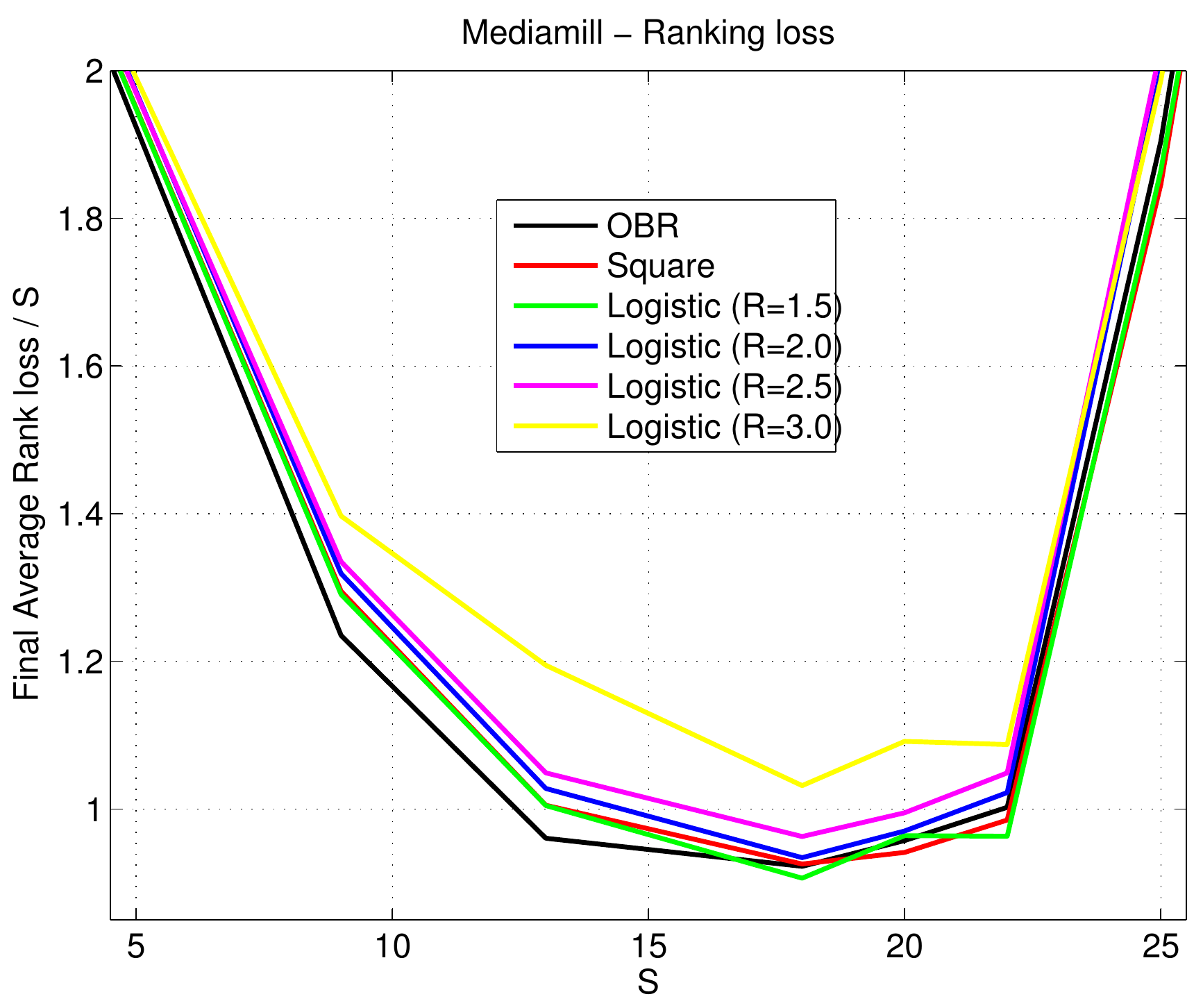}
  \end{tabular}
\caption{Experiments on the Mediamill dataset.\vspace{-0.15in}}
\label{fig:mediamill}
\end{figure}

\section{Technical details}\label{s:tech}
This section contains all proofs missing from the main text, along with ancillary
results and comments.

The algorithm in Figure \ref{f:2} works by updating through the 
gradients $\nabla_{i,t}$ of a 
modular margin-based loss function $\sum_{i=1}^K L(\bw_i^\top\bx)$
associated with the label generation model (\ref{e:labgenmult}), 
i.e., associated with function $g$, 
so as to make the parameters $(\bu_1,\ldots,\bu_K) \in \R^{dK}$ therein achieve 
the Bayes optimality condition
\begin{equation}\label{e:bayes_cond_mult}
(\bu_1,\ldots,\bu_K) = \arg\min_{ \bw_1,\ldots,\bw_K\,:\, \bw_i^\top\bx_t \in D} 
\E_t\left[\sum_{i=1}^K L(s_{i,t}\,\bw_{i}^\top \bx_t)\right]~,
\end{equation}
where $\E_t[\cdot]$ above is over the generation of $Y_t$ 
in producing the sign value $s_{i,t} \in \{-1,0,+1\}$, 
conditioned on the past (in particular, conditioned on $\hY_t$). 
The requirement in (\ref{e:bayes_cond_mult}) is akin to the classical construction of 
{\em proper scoring rules} in the statistical literature (e.g., \cite{sa73}).

The above is combined with the ability of the algorithm to guarantee
the high probability convergence of the prototype vectors $\bw'_{i,t}$ to the
corresponding $\bu_i$ (Lemma \ref{l:upperconfsingle}). The rate of convergence
is ruled by the fact that the associated upper confidence values $\epsilon_{i,t}$
shrink to zero as $\frac{1}{\sqrt{t}}$ when $t$ grows large. In order for this 
convergence to take place, it is important to insure that the algorithm is observing
informative feedback (either ``correct", i.e., $s_{i,t} = 1$, or ``mistaken", i.e., $s_{i,t} = -1$) 
for each class $i$ contained in the selected $\hY_t$. This in turn implies
regret bounds for both $\ell_{a,c}$ (Lemma \ref{l:onestepsingle}) 
and $\ell_{rank,t}$ (Lemma \ref{l:onestepsingle_ranking}).

The following lemma faces the problem of hand-crafting a convenient 
loss function $L(\cdot)$ such that (\ref{e:bayes_cond_mult}) holds. 
%
%
\begin{lemma}\label{l:expectation}
Let $\bw_1, \ldots, \bw_K \in \R^{dK}$ be arbitrary weight vectors such that 
$\bw_i^\top\bx_t \in D$, $i \in [K]$, $(\bu_1,\ldots,\bu_K) \in \R^{dK}$ be defined in (\ref{e:labgenmult}),
$s_{i,t}$ be the updating signs computed by the algorithm at the end (Step 5) of time $t$,
$L\,:\, D = [-R,R] \subseteq \R \rightarrow \R^+$ be a convex and differentiable function of its argument,
with $g(\Delta) = - L'(\Delta)$. Then for any $t$ we have
\[
\E_t\left[\sum_{i = 1}^K L(s_{i,t}\,\bw_i^\top\bx_t)\right] 
\geq 
\E_t\left[\sum_{i = 1}^K L(s_{i,t}\,\bu_i^\top\bx_t)\right],
\]
i.e., (\ref{e:bayes_cond_mult}) holds.
\end{lemma}
\begin{proof} 
Let us introduce the shorthands $\Delta_i = \bu_i^\top\bx_t$, $\hDelta_i =  \bw_{i,t}^\top\bx_t$,
$s_i = s_{i,t}$, and $p_i = \Pr(y_{i,t} = 1\,|\,\bx_t) = \frac{L'(-\Delta_i)}{L'(\Delta_i)+L'(-\Delta_i)}$.
Moreover, let $\Pr_t(\cdot)$ be an abbreviation for the conditional probability
$\Pr(\cdot\,|\,(y_1, \bx_1),\ldots, (y_{t-1},\bx_{t-1}),\bx_t)$.
Recalling the way $s_{i,t}$ is constructed (Figure \ref{f:2}), we can write
\begin{align*}
\E_t\left[\sum_{i = 1}^K L(s_{i,t}\,\hDelta_i)\right] 
&= \sum_{i \in \hY_t} \left(\Pr_t(s_{i,t}= 1)\,L(\hDelta_i) +
\Pr_t(s_{i,t}= -1)\,L(-\hDelta_i)
\right) + (K - |\hY_t|)\,L(0)\\
&= \sum_{i \in \hY_t} \left(p_i\,L(\hDelta_i)+(1-p_i)\,L(-\hDelta_i)
\right) + (K - |\hY_t|)\,L(0)~,
\end{align*}
For similar reasons,
\[
\E_t\left[\sum_{i = 1}^K L(s_{i,t}\,\Delta_i)\right] 
= \sum_{i \in \hY_t} \left(p_i\,L(\Delta_i)+(1-p_i)\,L(-\Delta_i)
\right) + (K - |\hY_t|)\,L(0)~.
\]
Since $L(\cdot)$ is convex, so is $\E_t\left[\sum_{i = 1}^K L(s_{i,t}\,\hDelta_i)\right] $ 
when viewed as a function of the $\hDelta_i$. We have that
$\frac{\partial\, \E_t\left[\sum_{i = 1}^K L(s_{i,t}\,\hDelta_i)\right] }{\partial \hDelta_i} = 0$ 
if and only if for all $i \in \hY_t$ we have that
$\hDelta_i$ satisfies
\[
p_i = \frac{L'(-\hDelta_i)}{L'(\hDelta_i)+L'(-\hDelta_i)}~.
\]
Since $p_i = \frac{L'(-\Delta_i)}{L'(\Delta_i)+L'(-\Delta_i)}$,
we have that
$\E_t\left[\sum_{i = 1}^K L(s_{i,t}\,\hDelta_i)\right]$ is minimized when $\hDelta_i = \Delta_i$
for all $i \in [K]$.
The claimed result immediately follows. 
\end{proof}

Let now $Var_t(\cdot)$ be a shorthand for  $Var(\cdot\,|\,(y_1, \bx_1),\ldots, (y_{t-1},\bx_{t-1}),\bx_t)$.
The following lemma shows that under additional assumptions on the loss $L(\cdot)$, we are afforded
to bound the variance of a difference of losses $L(\cdot)$ by the expectation of this
difference. This will be key to proving the fast rates of convergence contained in the
subsequent Lemma \ref{l:upperconfsingle}.
\begin{lemma}\label{l:variancesingle}
Let $(\bw'_{1,t},\ldots, \bw'_{K,t}) \in \R^{dK}$ be the weight vectors 
computed by the algorithm in Figure \ref{f:2} at 
the beginning (Step 2) of time $t$, $s_{i,t}$ be the updating signs
computed at the end (Step 5) of time $t$, and $(\bu_1, \ldots, \bu_K) \in \R^{dK}$ be the comparison vectors 
defined through (\ref{e:labgenmult}).
Let $L\,:\, D = [-R,R]\subseteq \R \rightarrow \R^+$ be a $C^2(D)$ convex function of its argument,
with $g(\Delta) = - L'(\Delta)$ and such that there are positive constants $c'_L$ and $c''_L$
with $(L'(\Delta))^2 \leq c'_L$ and $L''(\Delta) \geq c''_L$ for all $\Delta \in D$. Then for any $i \in \hY_t$
\[
0 \leq Var_t\left(L(s_{i,t}\,\bx_t^\top\bw'_{i,t}) - L(s_{i,t}\,\bu_i^\top\bx_t)\right)
\leq  \frac{2c'_L}{c''_L} \, \E_t\left[L(s_{i,t}\,\bx_t^\top\bw'_{i,t}) - L(s_{i,t}\,\bu_i^\top\bx_t)\right]~.
\]
\end{lemma}
\begin{proof}
Let us introduce the shorthands $\Delta_i = \bx_t^\top\bu_i$, $\hDelta_i =  \bx_t^\top\bw'_{i,t}$,
$s_i = s_{i,t}$, and recall that
$p_i = \Pr(y_{i,t} = 1\,|\,\bx_t) = \frac{L'(-\Delta_i)}{L'(\Delta_i)+L'(-\Delta_i)}$.
Then, for any $i \in [K]$,
\begin{align}
Var_t\left( L(s_{i,t}\,\bx_t^\top\bw'_{i,t}) - L(s_{i,t}\,\bu_i^\top\bx_t)\right)
&\leq \E_t\left(\left(L(s_{i}\,\hDelta_i) -  L(s_{i}\,\Delta_i)\right)^2\right)\leq c'_L\,(\hDelta_i-\Delta_i)^2~. \label{e:variancesingle}
\end{align}
Moreover, for any $i \in \hY_t$ we can write
\begin{align}
\E_t\left[L(s_{i}\,\hDelta_i) -  L(s_{i}\,\Delta_i)\right] 
&=
p_i\,(L(\hDelta_i) - L(\Delta_i) ) + (1-p_i)\,(L(-\hDelta_i) - L(-\Delta_i))\notag\\
&\geq p_i\,\left(L'(\Delta_i)(\hDelta_i-\Delta_i) + \frac{c''_L}{2}(\hDelta_i-\Delta_i)^2\right)\notag\\ 
& \ \ \ + (1-p_i)\,\left(L'(-\Delta_i)(\Delta_i-\hDelta) + \frac{c''_L}{2}(\hDelta_i-\Delta_i)^2\right)\notag\\
&= 
p_i\,\frac{c''_L}{2}(\hDelta_i-\Delta_i)^2 + (1-p_i)\,\frac{c''_L}{2}(\hDelta_i-\Delta_i)^2\notag\\
&= 
\frac{c''_L}{2}(\hDelta_i-\Delta_i)^2,\label{e:meansingle}
\end{align}
where the second equality uses the definition of $p_i$. Combining (\ref{e:variancesingle}) with (\ref{e:meansingle}) gives the desired bound. 
\end{proof}

We continue by showing a one-step regret bound {\em for our original} loss $\ell_{a,c}$. 
The precise connection to loss $L(\cdot)$ will be established with the help of a later 
lemma (Lemma \ref{l:upperconfsingle}).
\begin{lemma}\label{l:onestepsingle}
Let $L\,:\, D = [-R,R] \subseteq \R \rightarrow \R^+$ be a convex, twice differentiable, and nonincreasing 
function of its argument. 
Let $(\bu_1,\ldots,\bu_K) \in \R^{dK}$ be defined in (\ref{e:labgenmult}) with 
$g(\Delta) = - L'(\Delta)$ for all $\Delta \in D$.
Let also $c_L$ be a positive constant such that
\[
\frac{L'(\Delta)\,L''(-\Delta) + L''(\Delta)\,L'(-\Delta)}{(L'(\Delta) + L'(-\Delta))^2} \geq -c_L
\]
holds for all $\Delta \in D$.
Finally, let $\Delta_{i,t}$ denote $\bu_i^\top\bx_t$, and $\hDelta'_{i,t}$ denote $\bx_t^{\top}{{\bw'_{i,t}}}$,
where $\bw'_{i,t}$ is the $i$-the weight vector computed by the algorithm at the beginning (Step 2) of time $t$.
If time $t$ is such that 
$|\Delta_{i,t} -\hDelta'_{i,t}| \leq \epsilon_{i,t}$ for all $i \in [K]$,
then 
\[
\E_t[\ell_{a,c}(Y_t,\hY_t)] - \E_t[\ell_{a,c}(Y_t,Y^*_t)] 
\leq 
2\,(1-a)\,c_L\,\sum_{i \in \hY_t} \epsilon_{i,t}~.
\]
\end{lemma}
\begin{proof}
Recall the shorthand notation $p(\Delta) = \frac{g(-\Delta)}{g(\Delta)+g(-\Delta)}$. 
We can write
\begin{align*}
\E_t[\ell_{a,c}(Y_t,\hY_t)] &- \E_t[\ell_{a,c}(Y_t,Y^*_t)]\\
&= (1-a)\,\sum_{i \in \hY_t} 
\left( c({\hat j_i},|\hY_t|) - \left(\sfrac{a}{1-a} + c({\hat j_i},|\hY_t|)\right)\,p(\Delta_{i,t}) \right)\\
&\ \ - (1-a)\,\sum_{i \in Y^*_t} 
\left( c(j^*_i,|Y^*_t|) - \left(\sfrac{a}{1-a} + c(j^*_i,|Y^*_t|)\right)\,p(\Delta_{i,t}) \right),
\end{align*}
where ${\hat j_i}$ denotes the position of class $i$ in $\hY_t$ and $j^*_i$ is the position of class $i$ in $Y^*_t$.
Now, 
\[
p'(\Delta) 
= \frac{-g'(-\Delta)\,g(\Delta) - g'(\Delta)\,g(-\Delta)}{(g(\Delta)+g(-\Delta))^2}
= \frac{-L'(\Delta)\,L''(-\Delta) - L'(-\Delta)\,L''(\Delta)}{(L'(\Delta)+L'(-\Delta))^2} \geq 0
\]
since $g(\Delta) = -L'(\Delta)$, and $L(\cdot)$ is convex and nonincreasing. Hence $p(\Delta)$
is itself a nondecreasing function of $\Delta$. Moreover, the 
extra condition on $L$ involving $L'$ and $L''$ is a Lipschitz condition on $p(\Delta)$ via a
uniform bound on $p'(\Delta)$. Hence,
from $|\Delta_{i,t} -\hDelta'_{i,t}| \leq \epsilon_{i,t}$ and the definition of 
$\hY_t$ we can write 
\begin{align*}
\E_t[\ell_{a,c}(Y_t,\hY_t)] &- \E_t[\ell_{a,c}(Y_t,Y^*_t)]\\
&\leq  
(1-a)\,\sum_{i \in \hY_t} 
\left( c({\hat j_i},|\hY_t|)-\left(\sfrac{a}{1-a} 
                + c({\hat j_i},|\hY_t|)\right)\,p([\hDelta'_{i,t} - \epsilon_{i,t}]_D) \right)\\
&\ \ - (1-a)\,\sum_{i \in Y^*_t} 
\left( c(j^*_i,|Y^*_t|) - \left(\sfrac{a}{1-a} + c(j^*_i,|Y^*_t|)\right)\,p([\hDelta'_{i,t}+\epsilon_{i,t}]_D) \right)\\
&\leq  
(1-a)\,\sum_{i \in \hY_t} 
\left( c({\hat j_i},|\hY_t|)-\left(\sfrac{a}{1-a} 
                + c({\hat j_i},|\hY_t|)\right)\,p([\hDelta'_{i,t} - \epsilon_{i,t}]_D) \right)\\
&\ \ - (1-a)\,\sum_{i \in \hY_t} 
\left( c({\hat j_i},|\hY_t|)-\left(\sfrac{a}{1-a} + c({\hat j_i},|\hY_t|)\right)\,p([\hDelta'_{i,t}+\epsilon_{i,t}]_D) \right)\\
& = 
(1-a)\,\sum_{i \in \hY_t} 
\left( c({\hat j_i},|\hY_t|)\left(p([\hDelta'_{i,t}+\epsilon_{i,t}]_D)-p([\hDelta'_{i,t}-\epsilon_{i,t}]_D) \right)\right)\\
&\leq
2\,(1-a)\,c_L\,\sum_{i \in \hY_t} \epsilon_{i,t}~, 
\end{align*}
the last inequality deriving from $c(i,s)\leq 1$ for all $i \leq s \leq K$, and 
\[
p([\hDelta'_{i,t}+\epsilon_{i,t}]_D)-p([\hDelta'_{i,t}-\epsilon_{i,t}]_D) 
\leq 
c_L \bigl([\hDelta'_{i,t}+\epsilon_{i,t}]_D - [\hDelta'_{i,t}-\epsilon_{i,t}]_D\bigl) 
\leq
2\,c_L\,\epsilon_{i,t}.
\]
\end{proof}

Now, we first give a proof of Lemma \ref{l:bayesrank}, and then 
provide a one step regret for the partial information ranking loss.

\begin{proof}[Lemma \ref{l:bayesrank}]
Recall the notation $\Pr_t(\cdot) = \Pr(\cdot\,|\, \bx_t)$, and 
$p_{i,t} = p(\Delta_{i,t}) = \frac{g(-\Delta_{i,t})}{g(\Delta_{i,t})+g(-\Delta_{i,t})}$. 
For notational convenience, in this proof we drop subscript $t$ from $p_{i,t}$, $S_t$, $y_{i,t}$, $\hY_t$, 
and $\ell_{p-rank,t}$.
A simple adaptation of \cite{dwch12} (proof of Theorem 1 therein) shows that
for a generic sequence $\ha = (\ha_{1}, \ldots, \ha_K)$ 
with at most $S$ nonzero values $\ha_i$ and associated set of indices $\hY$, one has
\[
\E_t[\ell_{p-rank}(Y_t,\ha)] 
= \sum_{i, j \in \hY,\, i < j} \left(\hr_{i,j} + \hr_{j,i}\right) + 
S\,\left(\sum_{i \in [K]} p_{i} - \sum_{i \in \hY} p_{i} \right)
\]
where 
\[
\hr_{i,j} = \hr_{i,j}(\ha) = \Pr_t(y_{i} > y_{j})\,
\left(\{ \ha_{i} < \ha_{j} \} + \sfrac{1}{2}\,\{ \ha_{i} = \ha_{j} \} \right)\,.
\]
Moreover, if $p^*$ denotes the sequence made up of at most
$S$ nonzero values taken from $\{p_{i}\,,i \in [K]\}$, {\em where $i$ ranges again in} 
$\hY$, we have  
\[
\E_t[\ell_{p-rank}(Y_t,p^*)] 
= \sum_{i, j \in \hY,\, i < j} \left(r_{i,j} + r_{j,i}\right) +
S\,\left(\sum_{i \in [K]} p_{i} - \sum_{i \in \hY} p_{i} \right)
\]
with
\[
r_{i,j} = r_{i,j}(p^*)
= 
\Pr_t(y_{i} > y_{j})\,
\left(\{ p_{i} < p_{j} \} + \sfrac{1}{2}\,\{ p_{i} = p_{j} \} \right)\,.
\]
Hence
\[
\E_t[\ell_{p-rank}(Y_t,\ha)] - \E_t[\ell_{p-rank}(Y_t,p^*)] = 
\sum_{i, j \in \hY,\, i < j} \left(\hr_{i,j} - r_{i,j} + \hr_{j,i} - r_{j,i}\right).
\]
Since 
\[
\Pr_t(y_{i} > y_{j}) - \Pr_t(y_{j} > y_{i}) 
= \Pr_t(y_{i} = 1) - \Pr_t(y_{j} = 1) 
= p_{i} - p_{j},
\] 
a simple (but lengthy) case analysis reveals that
\[
\hr_{i,j} - r_{i,j} + \hr_{j,i} - r_{j,i} = 
\begin{cases}
\sfrac{1}{2}\,(p_{i} - p_{j})  
       &{\mbox{If $\ha_{i} < \ha_{j},\ p_{i} = p_{j}$ 
               or $\ha_{i} = \ha_{j},\ p_{i} > p_{j}$   }}\\
\sfrac{1}{2}\,(p_{j} - p_{i}) 
       &{\mbox{If $\ha_{i} = \ha_{j},\ p_{i} < p_{j}$ 
               or $\ha_{i} > \ha_{j},\ p_{i} = p_{j}$   }}\\ 
p_{i} - p_{j}
       &{\mbox{If $\ha_{i} < \ha_{j},\ p_{i} > p_{j}$ }}\\
p_{j} - p_{i}
       &{\mbox{If $\ha_{i} > \ha_{j},\ p_{i} < p_{j}$ }}~.
\end{cases}
\]
Notice that the above quantity is always nonnegative, and is strictly positive
if the $p_{i}$ are all different. 
The nonnegativity implies that {\em whatever set of indices $\hY$ we select}, the best
way to sort them within $\hY$ in order to minimize $\E_t[\ell_{p-rank}(Y_t,\cdot)]$
is by following the ordering of the corresponding $p_{i}$.

We are left to show that the best choice for $\hY$ is to collect the $S$ largest\footnote
{
It is at this point that we need the conditional independence assumption over the classes.
}
values in $\{p_{i}\,,i \in [K]\}$. 
To this effect, consider again $\E_t[\ell_{p-rank}(Y_t,p^*)] = \E_t[\ell_{p-rank}(Y_t,\hY)]$, 
and introduce the shorthand  
$p_{i,j} = p_i\,p_j = p_i - \Pr_t(y_{i} > y_{j})$.
Disregarding the term $S\,\sum_{i \in [K]} p_{i}$, which is independent
of $\hY$, we can write
\begin{align}
\E_t[\ell_{p-rank}(Y_t,\hY)] 
&=  \sum_{i, j \in \hY,\, i < j} \Pr_t(y_{i} > y_{j})\,
\left(\{ p_{i} < p_{j} \} + \sfrac{1}{2}\,\{ p_{i} = p_{j} \} \right)\notag\\ 
&\ \ \ \ \ +  \sum_{i, j \in \hY,\, i < j} \Pr_t(y_{j} > y_{i})\,
\left(\{ p_{j} < p_{i} \} + \sfrac{1}{2}\,\{ p_{j} = p_{i} \} \right) 
     - S\,\sum_{i \in \hY} p_{i} \notag\\
&=  \sum_{i, j \in \hY,\, i < j} (p_i-p_{i,j})\{ p_{i} < p_{j} \} +
              (p_i-p_{i,j}) \sfrac{1}{2}\,\{ p_{i} = p_{j} \} \notag\\
&\ \ \ \ \ +  \sum_{i, j \in \hY,\, i < j} (p_j-p_{i,j})\{ p_{j} < p_{i} \} +
              (p_j-p_{i,j}) \sfrac{1}{2}\,\{ p_{j} = p_{i} \} 
     - S\,\sum_{i \in \hY} p_{i}\notag\\
&=  \sum_{i, j \in \hY,\, i < j} (p_i-p_{j})\{ p_{i} < p_{j} \} +
              \sfrac{1}{2}\,(p_i-p_j)\,\{ p_{i} = p_{j} \} + p_j-p_{i,j}
     - S\,\sum_{i \in \hY} p_{i}\notag\\
&=  \sum_{i, j \in \hY,\, i < j} \left(\min\{p_i,p_j\} -p_i p_j\right) 
  - S\,\sum_{i \in \hY} p_{i}\notag
\end{align}
which can be finally seen to be equal to
\begin{equation}\label{e:minloss}
-\sum_{i \in \hY} (S+1-{\hat j}_i)\,p_i - \sum_{i, j \in \hY,\, i < j} p_i\,p_j\,,
\end{equation}
where
${\hat j}_i$ is the position of class $i$ within $\hY_t$ in decreasing order of $p_i$.

Now, rename the indices in $\hY$ as $1, 2, \ldots, S$, in such a way that $p_1 > p_2 > \ldots > p_S$
(so that ${\hat j}_i = i$), 
and consider the way to increase (\ref{e:minloss}) by adding to $\hY$ item
$k \notin \hY$ such that $p_S > p_k$ and removing from $\hY$ the item in position $\ell$.
Denote the resulting sequence by $\hY'$. From (\ref{e:minloss}), it is not hard to see that
\begin{align}
\E_t&[\ell_{p-rank}(Y_t,\hY)] -  \E_t[\ell_{p-rank}(Y_t,\hY')]\notag\\
&= (\ell-1)\,p_{\ell} + \sum_{i=\ell+1}^{S}p_{i} - \sum_{i=1}^{\ell-1}p_{i}\,p_{\ell} 
- \sum_{i=\ell+1}^{S}p_{\ell}\,p_{i} -(S-1)\,p_k 
+ \sum_{i=1,i \neq \ell}^{S}p_{i}\,p_{k} - S(p_{\ell}-p_k)\notag\\
&= (\ell-1)\,p_{\ell} + \sum_{i=\ell+1}^{S}p_{i} 
- (p_{\ell}-\,p_k)\,\sum_{i=1,i \neq \ell}^{S}p_{i} - (S-1)\,p_k -  S(p_{\ell}-p_k)\notag\\
&\leq (S-1)\,p_{\ell} - (p_{\ell}-\,p_k)\,\sum_{i=1,i \neq \ell}^{S}p_{i} 
- (S-1)\,p_k -  S(p_{\ell}-p_k)\notag\\
&= (p_{k}-\,p_{\ell})\left(1+\sum_{i=1,i \neq \ell}^{S}p_{i}\right)
\end{align}
which is smaller than zero since, by assumption, $p_{\ell} > p_k$.
Reversing the direction, if we maintain a sequence $\hY$ of size $S$, 
we can always reduce (\ref{e:minloss}) by removing its the last element and replacing
it with a larger element outside the sequence. We continue until no element
outside the current sequence exists which is larger than the smallest one in the sequence.
Clearly, we end up collecting the $S$ largest elements in $\{p_{i}\,,i \in [K]\}$.

Finally, from (\ref{e:minloss}) it is very clear that removing an element from 
a sequence $\hY$ with length $h \leq S$ can only increase the value of (\ref{e:minloss}).
Since this holds for an arbitrary $\hY$, and an arbitrary $h \leq S$ this shows,
that no matter which set $\hY$ we start off from, we always converge to the same
set containing exaclty the $S$ largest elements in $\{p_{i}\,,i \in [K]\}$.
This concludes the proof.
\end{proof}

\begin{lemma}\label{l:onestepsingle_ranking}
Under the same assumptions and notation as in Lemma \ref{l:onestepsingle}, combined
with the independence assumption (\ref{e:indep}), 
let the Algorithm in Figure \ref{f:2} be working with $a \rightarrow 1$ and strictly 
decreasing cost values $c(i,s)$, i.e., the algorithm is computing in round $t$ the ranking function
$\hf(\bx_t;S_t)$ defined in Section \ref{s:rank}. Let $\bw'_{i,t}$ be the $i$-th weight vector computed 
by this algorithm at the beginning (Step 2) of time $t$. If 
time $t$ is such that 
$|\Delta_{i,t} -\hDelta'_{i,t}| \leq \epsilon_{i,t}$ for all $i \in [K]$,
then 
\[
\E_t[\ell_{rank,t}(Y_t, \hf(\bx_t;S_t)] 
- \E_t[\ell_{rank,t}(Y_t,f^*(\bx_t;S_t)]
\leq 
4\,S_t\,c_L\,\sum_{i \in \hY_t} \epsilon_{i,t}~.
\]
\end{lemma}
\begin{proof}
We use the same notation as in the proof of Lemma \ref{l:bayesrank}, where
$\ha$ is now $\hY_t$, the sequence produced by ranking $\hf(\bx_t;S_t)$ 
operating on $\hp_{i,t}$.
Denote by $Y^*_t$ the sequences determined by $f^*(\bx_t;S_t)$, and let 
${\hat j}_i$ and ${j^*_i}$ be the position of class $i$ in decreasing order of $p_{i,t}$
within $\hY_t$ and $Y^*_t$, respectively.  

Proceeding as in Lemma \ref{l:onestepsingle} and recalling (\ref{e:minloss}) we can write
\begin{align*}
\E_t&[\ell_{p-rank,t}(Y_t,\hf(\bx_t;S_t))] 
- \E_t[\ell_{p-rank,t}(Y_t,f^*(\bx_t;S_t)]\\ 
& = \sum_{i \in Y^*_t} (S_t+1-j^*_i)\,p_i + \sum_{i, j \in Y^*_t,\, i < j} p_i\,p_j
- \sum_{i \in \hY_t} (S_t+1-{\hat j_i})\,p_i - \sum_{i, j \in \hY_t,\, i < j} p_i\,p_j\\
&\leq \sum_{i \in Y^*_t} (S_t+1-j^*_i)\,p([\hDelta'_{i,t}+\epsilon_{i,t}]_D) 
+ \sum_{i, j \in Y^*_t,\, i < j}
         p([\hDelta'_{i,t}+\epsilon_{i,t}]_D)\,p([\hDelta'_{j,t}+\epsilon_{j,t}]_D)\\
&\qquad - \sum_{i \in \hY_t} (S_t+1-{\hat j_i})\,p([\hDelta'_{i,t}-\epsilon_{i,t}]_D) 
- \sum_{i, j \in \hY_t,\, i < j} 
         p([\hDelta'_{i,t}-\epsilon_{i,t}]_D)\,p([\hDelta'_{j,t}-\epsilon_{j,t}]_D)\\
&\leq \sum_{i \in \hY_t} 
   (S_t+1-{\hat j_i})\,\left(p([\hDelta'_{i,t}+\epsilon_{i,t}]_D) 
                         - p([\hDelta'_{i,t}-\epsilon_{i,t}]_D) \right) \\
&\qquad + \sum_{i, j \in \hY_t,\, i < j}
      \left(p([\hDelta'_{i,t}+\epsilon_{i,t}]_D)\,p([\hDelta'_{j,t}+\epsilon_{j,t}]_D) - 
       p([\hDelta'_{i,t}-\epsilon_{i,t}]_D)\,p([\hDelta'_{j,t}-\epsilon_{j,t}]_D)\right)\\
&\leq 2S_t c_L\,\sum_{i \in \hY_t}\epsilon_{i,t}
+ \sum_{i, j \in \hY_t,\, i < j} 2c_L\,(\epsilon_{i,t}+\epsilon_{j,t})\\
&= 2\,S_t\,c_L\,\sum_{i \in \hY_t} \epsilon_{i,t} +  2\,(S_t-1)\,c_L\,\sum_{i \in \hY_t}\epsilon_{i,t} \\
&< 4\,S_t\,c_L\,\sum_{i \in \hY_t} \epsilon_{i,t}\,,
\end{align*}
as claimed.
\end{proof}

\begin{lemma}\label{l:upperconfsingle}
Let $L\,:\, D = [-R,R] \subseteq \R \rightarrow \R^+$ be a $C^2(D)$ convex and nonincreasing function of its 
argument,
$(\bu_1,\ldots,\bu_K) \in \R^{dK}$ be defined in (\ref{e:labgenmult}) with $g(\Delta) = - L'(\Delta)$ for all 
$\Delta \in D$, and such that $\|\bu_i\| \leq U$ for all $i \in [K]$. 
Assume there are positive constants $c'_L$ and $c''_L$
with $(L'(\Delta))^2 \leq c'_L$ and $L''(\Delta) \geq c''_L$ for all $\Delta \in D$.
With the notation introduced in Figure \ref{f:2}, we have that 
\begin{align*}
({\bx^\top\bw'_{i,t}}- \bu_i^\top \bx)^2 \leq \bx^\top A^{-1}_{i,t-1} \bx\,
\left(U^2 +\frac{d\,c'_L}{(c''_L)^2}\,\ln \left(1+ \frac{t-1}{d} \right)
+ \frac{12}{c''_L}\,\left(\frac{c'_L}{c''_L} + 3 L(-R)\right)\,\ln \frac{K(t+4)}{\delta} \right)
\end{align*}
holds with probability at least $1-\delta$ for any $\delta< 1/e$, {\em uniformly} over $i \in [K]$, 
$t = 1, 2, \ldots, $ and $\bx \in \R^{d}$.
\end{lemma}
\begin{proof}
For any given class $i$, the time-$t$ update rule $\bw'_{i,t} \rightarrow \bw_{i,t+1} \rightarrow \bw'_{i,t+1}$ 
in Figure \ref{f:2} allows us to start off from \cite{hka07} (proof of Theorem 2 therein), 
from which one can extract the following inequality
\begin{align}
&d_{i,t-1}(\bu_i,\bw'_{i,t}) \nonumber\\
&\quad \leq U^2 +\frac{1}{(c''_L)^2}\,\sum_{k = 1}^{t-1} r_{i,k}
- \frac{2}{c''_L}\,\sum_{k = 1}^{t-1} \left(\nabla_{i,k}^\top (\bw'_{i,k} - \bu_i) 
     - \frac{c''_L}{2}\,\left(s_{i,k}\,\bx_k^\top (\bw'_{i,k} - \bu_i)\right)^2  \right), \label{e:partialres}
\end{align}
where we set $r_{i,k} = \nabla_{i,k}^\top\,A_{i,k}^{-1}\,\nabla_{i,k}$.
Using the lower bound on the second derivative of $L$ we have
\begin{align*}
&L(s_{i,k}\,\bx_k^\top\bw'_{i,k}) - L(s_{i,k}\,\bu_i^\top \bx_k) \\
&\qquad \leq L'(s_{i,k}\,\bx_k^\top\bw'_{i,k})(s_{i,k}\bx_k^\top \bw'_{i,k}-s_{i,k}\,\bu_i^\top \bx_k) 
   - \frac{c''_L}{2} (s_{i,k}\,\bx_k^\top\bw'_{i,k} - s_{i,k}\,\bu_i^\top \bx_k)^2 \\
& \qquad =\nabla_{i,k}^\top (\bw'_{i,k} 
            - \bu_i) - \frac{c''_L}{2}\,\left(s_{i,k}\,\bx_k^\top (\bw'_{i,k} - \bu_i)\right)^2~.
\end{align*}
Plugging back into (\ref{e:partialres}) yields
\begin{equation}\label{e:partialres2}
d_{i,t-1}(\bu_i,\bw'_{i,t}) 
\leq 
U^2 +\frac{1}{(c''_L)^2}\,\sum_{k = 1}^{t-1} r_{i,k}
- \frac{2}{c''_L}\,\sum_{k = 1}^{t-1} \left( L(s_{i,k}\,\bx_k^\top\bw'_{i,k}) - L(s_{i,k}\,\bu_i^\top \bx_k) \right)
\end{equation}
We now borrow a proof technique from \cite{dgs10} (see also \cite{cg11,ayps11} and references therein).
Define 
\[
L_{i,k} =  L(s_{i,k}\,\bx_k^\top\bw'_{i,k}) - L(s_{i,k}\,\bu_i^\top \bx_k)
\]
and $L'_{i,k} = \E_k[L_{i,k}] - L_{i,k}$.
Notice that the sequence of random variables 
$L'_{i,1}$, $L'_{i,2}, \ldots ,$ forms a martingale difference sequence such that, for any $i \in \hY_k$:
\begin{enumerate}
\item [i.] $\E_k[L_{i,k}] \geq 0$, by Lemma \ref{l:variancesingle};  
\item [ii.] $|L'_{i,k}| \leq 2 L(-R)$, since $L(\cdot)$ is nonincreasing over $D$, and 
$s_{i,k}\,\bx_k^\top\bw'_{i,k}$, $s_{i,k}\,\bu_i^\top \bx_k \in D$; 
\item [iii.] $Var_k(L'_{i,k}) = Var_k(L_{i,k}) \leq \frac{2c'_L}{c''_L}\,\E_k[L_{i,k}]$ 
(again, because of Lemma \ref{l:variancesingle}).
\end{enumerate}
On the other hand, when $i \notin \hY_k$ then $s_{i,k} = 0$, and the above
three properties are trivally satisfied.
Under the above conditions, we are in a position to apply any fast concentration result for bounded
martingale difference sequences. For instance, setting for brevity
$B = B(t,\delta) =  3\,\ln \frac{K(t+4)}{\delta}$, a result contained in
\cite{kt08} allows us derive the inequality
\[
\sum_{k=1}^{t-1} \E_k[L_{i,k}] - \sum_{k=1}^{t-1} L_{i,k} 
\geq \max\left\{ \sqrt{\frac{8c'_L}{c''_L}\,B\,\sum_{k=1}^{t-1} \E_k[L_{i,k}]}, 6 L(-R)\,B \right\}~,
\]
that holds with probability at most $\frac{\delta}{Kt(t+1)}$ for any $t \geq 1$.
We use the inequality $\sqrt{cb} \leq \sfrac{1}{2}(c+b)$ with $c = \frac{4c'_L}{c''_L}\,B$, and
$b = 2\,\sum_{k=1}^{t-1} \E_k[L_{i,k}]$, and simplify. This gives 
\[
- \sum_{k=1}^{t-1} L_{i,k} \leq  \left(\frac{2 c'_L}{c''_L} + 6 L(-R)\right)\,B \\
\]
with probability at least $1- \frac{\delta}{Kt(t+1)}$.
Using the Cauchy-Schwarz inequality 
\[
({\bx^\top\bw'_{i,t}} -\bu_i^\top \bx)^2 \leq \bx^\top A^{-1}_{i,t-1}\, \bx\,d_{i,t-1}(\bu_i,\bw'_{i,t})
\]
holding for any $\bx \in \R^d$,
and replacing back into (\ref{e:partialres2}) allows us to conclude that

\begin{equation}\label{e:partialres3}
({\bx^\top\bw'_{i,t}}- \bu_i^\top \bx)^2 \leq \bx^\top A^{-1}_{i,t-1} \bx\,
\left(U^2 +\frac{1}{(c''_L)^2}\,\sum_{k = 1}^{t-1} r_{i,k}
+ \frac{12}{c''_L}\,\left(\frac{c'_L}{c''_L} + 3 L(-R)\right)\,\ln \frac{K(t+4)}{\delta} \right)
\end{equation}
holds with probability at least $1-\frac{\delta}{Kt(t+1)}$, {\em uniformly} over $\bx \in \R^{d}$.

The bounds on $\sum_{k = 1}^{t-1} r_{i,k}$ can be obtained in a standard way.
Applying known inequalities \cite{aw01,ccg02,cgo09,ccg11,hka07,dgs10}, 
and using the fact that $\nabla_{i,k} = L'(s_{i,k}\,\bx_k^\top \bw'_{i,k})\,s_{i,k}\bx_k$
we have\footnote
{
It is in this chain of inequalities that we exploit the rank-one update of $A_{i,t-1}$ based 
on $\bx_t\bx_t^\top$ rather than $\nabla_{i,t}\nabla_{i,t}^\top$. Notice that using the latter 
(as in the worst-case analysis by \cite{hka07}), does not guarantee a significant progress in 
the positive definiteness of $A_{i,t}$. This is due to the presence of the multiplicative factor 
$g(s_{i,t}\hDelta'_{i,t})$ (Step 5 in Figure \ref{f:2}) which can be arbitrarily small.
} 
\begin{eqnarray*}
\sum_{k = 1}^{t-1} r_{i,k} 
&=& \sum_{k = 1}^{t-1} |s_{i,j}|\,(L'(s_{i,k}\,\bx_k^\top \bw'_{i,k}))^2\,\bx_k^\top A^{-1}_{i,k} \bx_k \\
&\leq& c'_L\,\sum_{k = 1}^{t-1} |s_{i,k}| \bx_k^\top A^{-1}_{i,k} \bx_k \\
&\leq& c'_L\,\sum_{k = 1}^{t-1} \ln \frac{|A_{i,k}|}{|A_{i,k-1}|}\\ 
&=& c'_L\,\ln \frac{|A_{i,t-1}|}{|A_{i,0}|}\\ 
&\leq& d\,c'_L\,\ln \left(1+ \frac{t-1}{d} \right)~.
\end{eqnarray*}
Combining as in (\ref{e:partialres3}) and stratifying over $t = 1, 2, \ldots$, and $i \in [K]$  
concludes the proof. 
\end{proof}


We are now ready to put all pieces together.

\begin{proof}[Theorem \ref{t:cumregret}]
From Lemma~\ref{l:onestepsingle} and Lemma~\ref{l:upperconfsingle}, we see that
with probability at least $1-\delta$,
\begin{equation}\label{e:partial}
R_T \leq 2\,(1-a)\,c_L \,\sum_{t=1}^T \sum_{i \in \hY_t} \epsilon_{i,t}\,,
\end{equation}
when $\epsilon^2_{i,t}$ is the one given in Figure \ref{f:2}. 
We continue by proving a pointwise upper bound on the sum in the RHS.
More in detail, we will find an upper bound on $\sum_{t=1}^T \sum_{i \in \hY_t} \epsilon^2_{i,t}$,
and then derive a resulting upper bound on the RHS of (\ref{e:partial}).

From Lemma \ref{l:upperconfsingle} and the update rule (Step 5) 
of the algorithm we can write
%
%
\begin{align*}
\epsilon^2_{i,t} 
&\leq C\,\bx_{t}^\top A_{i,t-1}^{-1} \bx_{t}\\
&= C\,\frac{\bx_{t}^\top (A_{i,t-1}+|s_{i,t}|\,\bx_{t} \bx_{t}^\top)^{-1} \bx_{t}}{1-|s_{i,t}| \bx_t^\top (A_{i,t-1}+|s_{i,t}|\,\bx_t \bx_t^\top)^{-1} \bx_t}\\
&= C\,\frac{\bx_{t}^\top A_{i,t}^{-1} \bx_{t}}{1-|s_{i,t}| \bx_t^\top (A_{i,t-1}+|s_{i,t}|\,\bx_t \bx_t^\top)^{-1} \bx_t}\\
&\leq C\,\frac{\bx_{t}^\top A_{i,t}^{-1} \bx_{t}}{1-|s_{i,t}| \bx_t^\top (A_{0}+|s_{i,t}|\,\bx_t \bx_t^\top)^{-1} \bx_t}\\
&= C\,\frac{\bx_t^\top A_{i,t}^{-1} \bx_t}{1-\frac{1}{2}}\\
&= 2\,C\,\bx_t^\top A_{i,t}^{-1} \bx_t~.
\end{align*}
Hence, if we set $r_{i,t} = \bx_t^\top A_{i,t}^{-1} \bx_t $ and proceed as in the proof of 
Lemma \ref{l:upperconfsingle}, we end up with the upper bound  
$\sum_{t = 1}^T \epsilon^2_{i,t} \leq 2\,C\,d\,\ln \left(1+ \frac{T}{d} \right)$, holding for all $i \in [K]$.
%
Denoting by $M$ the quantity $2\,C\,d\,\ln \left(1+ \frac{T}{d} \right)$,
we conclude from (\ref{e:partial}) that
\[
R_T 
\leq 
2\,(1-a)\,c_L\, \max\left\{ \sum_{i \in [K]} \sum_{t=1}^T \epsilon_{i,t}
\,\Bigl|\, 
\sum_{t = 1}^T \epsilon^2_{i,t} \leq M,\,\,\, i \in [K] \right\} = 2\,(1-a)\,c_L\,K\,\sqrt{T\,M}~,
\]
as claimed. 
\end{proof}

\begin{proof}[Theorem \ref{t:cumregret_logt}]
As we said, we change the definition of $\epsilon^2_{i,t}$ 
in the Algorithm in Figure \ref{f:2} to
\begin{align*}
&\epsilon^2_{i,t}=\\
& \max\Biggl\{\bx^\top A^{-1}_{i,t-1} \bx\,
\left(\frac{2\,d\,c'_L}{(c''_L)^2}\,\ln \left(1+ \frac{t-1}{d} \right)
+ \frac{12}{c''_L}\,\left(\frac{c'_L}{c''_L} + 3 L(-R)\right)\,\ln \frac{K(t+4)}{\delta} \right), 4\,R^2\Biggl\}\,.\notag
\end{align*}
\sloppypar{
First, notice that the $4R^2$ cap seamlessly applies, 
since $({\bx^\top\bw'_{i,t}}- \bu_i^\top \bx)^2$ in Lemma \ref{l:upperconfsingle} is bounded by 
$4\,R^2$ anyway. With this modification, we have that Theorem \ref{t:cumregret} only holds
for $t$ such that $\frac{d\,c'_L}{(c''_L)^2}\,\ln \left(1+ \frac{t-1}{d} \right)\geq U^2$, i.e.,
for $t \geq d\,\left(\exp\left(\frac{(c''_L)^2\,U^2}{c'_L\,d}\right)-1 \right)+1$,
while for 
$t < d\,\left(\exp\left(\frac{(c''_L)^2\,U^2}{c'_L\,d}\right)-1 \right)+1$ 
we have in the worst-case scenario the maximum 
amount of regret at each step. From Lemma \ref{l:onestepsingle} we see that this maximum amount
(the cap on $\epsilon^2_{i,t}$ is needed here)
can be bounded by $4\,(1-a)\,c_L\,|\hY_t|\,R \leq 4\,(1-a)\,c_L\,K\,R$. 
}
\end{proof}

\begin{proof}[Theorem \ref{t:cumregret_rank}]
We start from the one step-regret delivered by Lemma \ref{l:onestepsingle_ranking},
and proceed as in the proof of Theorem \ref{t:cumregret}. This yields
\begin{align*}
R_T 
& \leq 4\,c_L \,\sum_{t=1}^T S_t\,\sum_{i \in \hY_t} \epsilon_{i,t} \\
&\leq 4\,S\,c_L \,\sum_{t=1}^T \sum_{i \in \hY_t} \epsilon_{i,t} \\
&\leq 4\,S\,c_L \,\sum_{t=1}^T \sum_{i \in [K]} \epsilon_{i,t}\\
&= 4\,S\,c_L \,\sum_{i \in [K]} \sum_{t=1}^T \epsilon_{i,t}\,,
\end{align*}
with probability at least $1-\delta$, where $\epsilon^2_{i,t}$ is the one given in Figure \ref{f:2}. Let $M$ be as in the proof of Theorem \ref{t:cumregret}. 
If $N_{i,T}$ denotes the total number of times class $i$ occurs in
$\hY_t$, we have that  $\sum_{t=1}^T \epsilon^2_{i,t} \leq M$, 
implying $\sum_{t=1}^T \epsilon_{i,t} \leq \sqrt{N_{i,T}\,M}$ for all $i \in [K]$.
Moreover, $\sum_{i \in [K]} N_{i,T} \leq ST$. Hence
\[
R_T \leq 4\,S\,c_L \,\sum_{i \in K]} \sqrt{N_{i,T}\,M} \leq 4\,c_L\,\sqrt {M\,S\,K\,T}\,,
\]
as claimed. 
\end{proof}

\section{Conclusions}\label{s:concl}
We have used generalized linear models to formalize the exploration-exploitation
tradeoff in a multilabel/ranking setting with partial feedback, providing 
$T^{1/2}$-like regret bounds under semi-adversarial settings. Our analysis decouples
the multilabel/ranking loss at hand from the label-generation model.
Thanks to the usage of calibrated score values $\hp_{i,t}$, our algorithm
is capable of automatically inferring where to split the ranking between relevant
and nonrelevant classes \cite{fhlmb08}, the split being clearly induced by
the loss parameters in $\ell_{a,c}$.
We are planning on using more general label models that explicitly capture
label correlations to be applied to other loss functions (e.g., F-measure, 0/1, 
average precision, etc.). We are also planning on carrying out a 
more thorough experimental comparison, especially to full information multilabel
methods that take such correlations into account. Finally, we are currenty working
on extending our framework to structured output tasks, like (multilabel) hierarchical 
classification.

\end{document}